\documentclass{article} %
\ifdefined\pdfsuppressptexinfo\pdfsuppressptexinfo=15\fi
\PassOptionsToPackage{sort}{natbib}
\usepackage{iclr2027_conference,times}
\usepackage{amsthm}
\newtheorem{proposition}{Proposition}
\newtheorem*{propositionrestated}{Proposition~\ref*{prop:output-gradient}} %

\usepackage{amsmath,amsfonts,bm}

\def\eqref#1{equation~\ref{#1}}

\def\1{\bm{1}}

\DeclareMathAlphabet{\mathsfit}{\encodingdefault}{\sfdefault}{m}{sl}
\SetMathAlphabet{\mathsfit}{bold}{\encodingdefault}{\sfdefault}{bx}{n}

\usepackage{url}
\usepackage{booktabs} 
\usepackage{graphicx}
\usepackage{placeins}
\usepackage{algorithm} 
\usepackage{algpseudocode} 
\usepackage{xcolor}  
\usepackage{hyperref}

\usepackage{mathtools}

\usepackage{enumitem}
\setlist[itemize]{
  itemsep=8pt,
  parsep=0pt,
  topsep=2pt,
  partopsep=0pt,
  leftmargin=1.5em,
  labelsep=0.5em
}

\setlist[enumerate]{
  itemsep=0pt,
  parsep=0pt,
  topsep=2pt,
  partopsep=0pt,
  leftmargin=1.5em,
  labelsep=0.5em
}

\usepackage[capitalize,noabbrev]{cleveref}

\makeatletter
\newcommand{\sublabel}[2]{\protected@edef\@currentlabel{\thetable#1}\label{#2}}
\newcommand{\figsublabel}[2]{\protected@edef\@currentlabel{\thefigure#1}\label{#2}}
\makeatother

\usepackage{xcolor}
\definecolor{findingblue}{HTML}{2A7F62}

\newcounter{researchfinding}
\newcommand{\researchfinding}[2]{%
  \begingroup
  \refstepcounter{researchfinding}%
  \label{#1}%
  \endgroup
  \textit{\textcolor{findingblue}{Finding \theresearchfinding: #2}}%
}

\title{LESSER: Post-Training Data Selection \\ with Output-Layer Gradients}

\author{Lyuxin David Zhang$^{1}$ \quad Eric Wong$^{1}$ \quad
Surbhi Goel$^{1}$ \quad Anton Xue$^{2}$ \\
$^{1}$University of Pennsylvania \\
$^{2}$University of Texas at Austin
}

\iclrfinalcopy

\newcommand{\LESSER}{\textsc{LESSER}}

\newcommand{\RDS}{\textsc{RDS+}}
\newcommand{\LESS}{\textsc{LESS}}

\begin{document}

\maketitle
\lhead{Preprint. Under review.}

\begin{abstract}

The choice of post-training data for large language models substantially affects downstream performance.
Gradient-based data selection is a popular approach that ranks training data by how well their gradients align with those of a small validation set.
However, ranking with full-parameter gradients requires an expensive backward pass on every sample, making computation intractable for large candidate pools.
This raises a natural question: can we approximate full-gradient features at a fraction of the cost?
Conveniently, we find that output-layer gradients suffice for effective data selection, yet require only the cheaper forward pass.
We implement this as LESSER, a drop-in wrapper for selection methods that reduces the feature-extraction FLOP cost by $9.7\times$ for SFT and $3.0\times$ for RL benchmarks, while tracking full-gradient performance on downstream tasks. 
Empirically, we find that even when output-layer and full gradients rank individual samples differently, they select batches with aligned gradients.

\end{abstract}

\begin{figure}[H]
\centering
\includegraphics[width=0.9\linewidth,height=0.4\textheight,keepaspectratio]{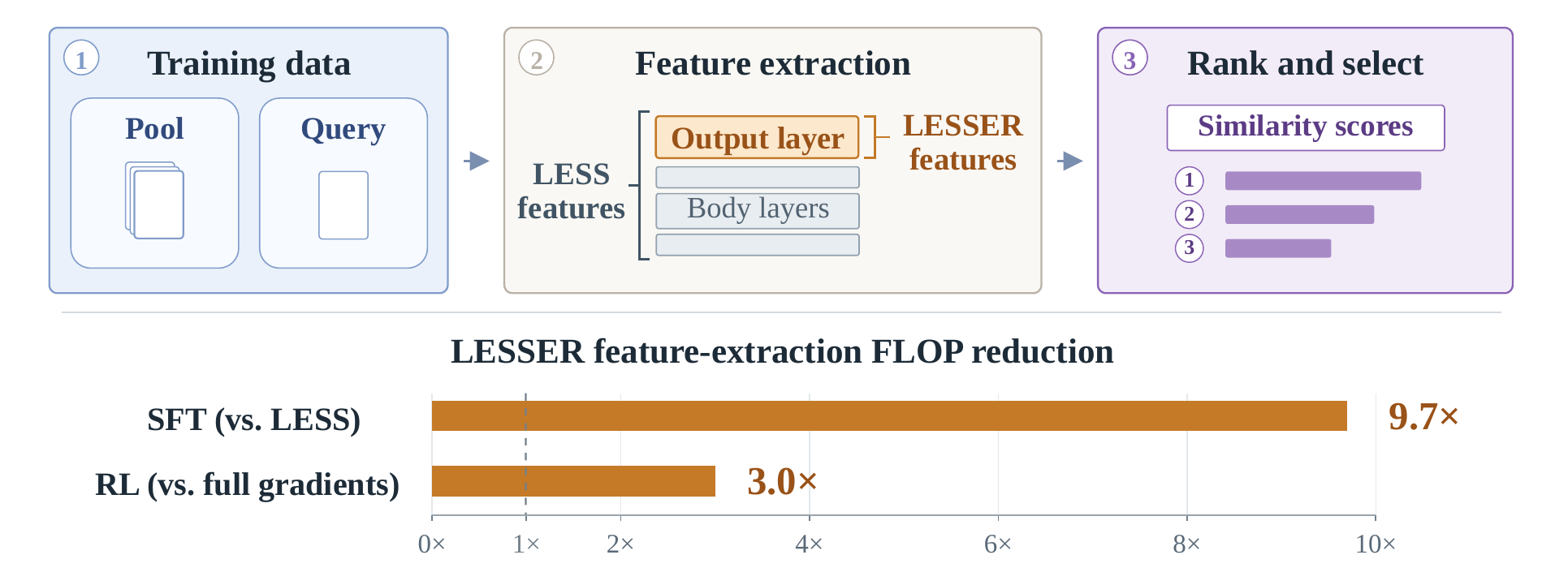}
\caption{
\textbf{Efficient data selection with output-layer gradients.}
Given a candidate training pool and a small query set representative of downstream tasks, \LESSER{} extracts output-layer gradient features to rank and select training examples.
These features approximate the downstream behavior of full-gradient methods like \LESS{} while substantially reducing feature-extraction FLOPs.
}
\label{fig:teaser}
\end{figure}

\section{Introduction}

The choice of post-training data substantially shapes the performance of large language models (LLMs)~\citep{ouyang2022training,wang2023self,guo2025deepseek}.
In particular, training on small, high-quality subsets tailored to a target task can match or even outperform training on a larger, general pool~\citep{xia2024less,muennighoff2025s1}.
However, identifying which subset to use is non-trivial, as the utility of each sample depends on the model and target task, and selection can become expensive when the candidate pool grows.
These challenges have motivated the development of data selection methods for identifying targeted, high-quality subsets~\citep{killamsetty2021glister}.

A common approach to targeted data selection maps candidate and query examples into a shared feature space and selects candidates by their similarity to the queries~\citep{zhang2025survey}.
When choosing features, gradients are a natural choice because they describe how each training sample would update the model.
This principle underlies modern LLM data selection methods such as \LESS{}~\citep{xia2024less} and GIST~\citep{min2026gist} for SFT, GradAlign~\citep{yang2026gradalign} for RL, and GRACE~\citep{panigrahi2025good} for distillation teacher selection.

However, full-parameter gradients are costly to compute at scale: each candidate requires a backward pass through the LLM.
For large candidate pools, extraction can cost more than training on the selected data itself~\citep{yin2024compute}.
This has motivated more compute-efficient features, such as \RDS{}~\citep{ivison2025large}, which uses hidden-state features, and proxy-model methods such as \citet{yang2024smalltolarge}, which use smaller models to guide selection.
Despite their cost, however, gradient-based features often yield superior subsets~\citep{nayak2026critical}.
We thus aim to approximate gradients at a fraction of the cost.

In particular, output-layer gradients are attractive because they can be computed directly from last-layer hidden states and output probabilities available during a forward pass, without backpropagating through the model.
These features have also proved useful in other settings, including active learning for vision models~\citep{ash2019deep}, data valuation~\citep{deng2025efficient}, and data attribution~\citep{ran2026sketching} methods.
We find that they are also effective for post-training data selection.

\subsection{Our contributions}

We propose \textbf{\LESSER{}} (\textbf{L}ast-layer \textbf{E}rror \textbf{S}ignals for \textbf{S}calable \textbf{E}xample \textbf{R}anking), a drop-in wrapper for gradient-based data selection methods that only swaps in output-layer gradients, leaving the selection rule unchanged.
We implement this for~\LESS{} in SFT data selection, for GradAlign in RL problem selection, and for GRACE in distillation teacher selection.

\paragraph{\LESSER{} is cost-efficient.}
Extracting output-layer gradients is cost-efficient, requiring $9.7\times$ fewer FLOPs on SFT and $3.0\times$ fewer FLOPs on RL.
This effect is more pronounced in our wall-clock experiments (Section~\ref{sec:cost}), where \LESSER{} reduces feature-extraction time by $17.8\times$ in SFT, the added scoring time per selection round by $3.5\times$ in RL, and feature-extraction time by $3.9$--$5.4\times$ in distillation teacher selection.

\paragraph{\LESSER{} tracks full-gradient performance.}
Across a variety of setups, output-layer gradients find subsets that reach downstream performance similar to that of subsets found with full gradients.

\begin{itemize}
    \item \textbf{SFT.} Across five SFT tasks and four LLMs, the absolute downstream score gap between \LESSER{} and \LESS{} averages only $1.3$ points, growing to $2.2$ against \RDS{}, and to $2.6$ against a random selection baseline.
    Relative to \LESS{}-selected batches, the centered gradient cosine similarity is \(0.29\) for \LESSER{}, compared with \(0.03\) for \RDS{} and approximately zero for Random.
    At the base model, the batch gradients of \LESSER{} selections thus align with those of \LESS{} more than other baselines do.

    \item \textbf{RL}. 
    In RL, the problems \LESSER{} and GradAlign select in the first round have a Jaccard similarity of $0.53$, compared to the random baseline's $0.075$. After training, \LESSER{} and GradAlign reach the same final test accuracy.

    \item \textbf{Distillation.}
    Across three different distillation settings, \LESSER{} consistently selects the same top teacher from a pool of $10$ to $14$ teachers as the unmodified GRACE.
    This suggests that GRACE's teacher choice is robust to replacing full gradients with output-layer gradients.

\end{itemize}

\subsection{Related work}
\label{sec:related}

\paragraph{Targeted data selection.}
Data selection methods differ in the notion of usefulness they aim to capture.
Quality filtering and diversity-based curation seek examples that are broadly useful for instruction tuning~\citep{chen2023alpagasus,zhou2023lima,liu2024makes}.
Other methods select data relative to a reference distribution or dataset: DSIR~\citep{xie2023data} matches a target text distribution for pretraining, while CRAIG~\citep{mirzasoleiman2020coresets} constructs weighted subsets that approximate the full training-set gradient.
Selection can also be tied directly to a training objective: GLISTER~\citep{killamsetty2021glister} optimizes held-out validation performance, while GRAD-MATCH~\citep{killamsetty2021grad} matches the gradient of the training or validation set.
For targeted post-training,
\LESS{}~\citep{xia2024less} uses gradient alignment and \RDS{}~\citep{ivison2025large} uses hidden-state similarity.
GIST~\citep{min2026gist} projects gradients into a subspace defined by validation gradients, while GradAlign~\citep{yang2026gradalign} uses policy-gradient alignment for online RL selection.
Systematic evaluations show that selection gains vary across datasets and budgets~\citep{diddee2025chasing,nayak2026critical}, and that the cost of selection can outweigh its training gains~\citep{yin2024compute}.
We instead focus on the efficiency of data selection pipelines, and show that output-layer gradients are a cost-efficient substitute for full-gradient features.

\paragraph{Efficient gradient representations.}
Gradient-based attribution connects training examples to their effects on model predictions.
\citet{koh2017understanding} use influence functions to estimate the effect of upweighting a training example, while TracIn accumulates gradient inner products across training checkpoints~\citep{pruthi2020estimating}.
To make attribution scalable, TRAK~\citep{park2023trak} combines model linearization with random projection, while LoGra~\citep{choe2026your} exploits the outer-product structure of layer gradients to reduce projection costs.
These approaches compress gradient features but still require backpropagation.
Restricting the gradient to the last layer is another approach, used in GLISTER~\citep{killamsetty2021glister} and GRAD-MATCH~\citep{killamsetty2021grad} for subset selection and in BADGE~\citep{ash2019deep} for active learning.

Closely related, For-Value~\citep{deng2025efficient} derives a forward-only valuation score from output-gradient inner products, while RISE~\citep{ran2026sketching} combines sketches of this signal with an additional feature that maps prediction residuals into hidden-state space.
Both methods target valuation or attribution, which score individual examples, and treat the output-layer score as the value itself or as an approximation of full-gradient influence, respectively.
We instead keep existing SFT, RL, and distillation teacher selectors~\citep{panigrahi2025good} unchanged, replace only their gradient features, and evaluate the selections downstream.
In SFT especially, output-layer and full gradients rank individual examples differently, yet their selections perform similarly downstream.
This differs from the per-example view taken by prior work, under which weaker agreement on individual examples should yield weaker selections; Section~\ref{sec:analysis-new} explains it through the aligned gradients of the selected batches.

\section{Background}
\label{sec:background}

\subsection{Feature-based data selection}
\label{sec:background:selection}

We consider targeted data selection from a training pool
$\mathcal D=\{p_1,\dots,p_N\}$ using a small query set
$\mathcal Q=\{q_1,\dots,q_M\}$ drawn from the target test distribution.
We refer to examples in $\mathcal Q$ as \emph{queries}, since they specify the target task and guide which training examples are selected.
Given a budget $k$, the goal is to choose a subset $S\subseteq\mathcal D$ of size $k$ such that training on $S$ improves performance on the target distribution, with $\mathcal Q$ serving as the available proxy for this target during selection.

Which subset $S$ should we choose?
Many targeted data selection methods answer this question using a feature-based pipeline.
Each candidate $p\in\mathcal D$ and query $q\in\mathcal Q$ is mapped to a feature vector $\phi_p$ and $\phi_q$ in a shared representation space.
A similarity function $\operatorname{sim}(\phi_p,\phi_q)$ then measures how relevant candidate $p$ is to query $q$.
This separates two design choices: the feature representation and the selection rule, which together are used to extract a size-$k$ training subset from the pool.

\begin{figure}[!t]
\centering
\includegraphics[width=\linewidth,height=0.28\textheight,keepaspectratio]{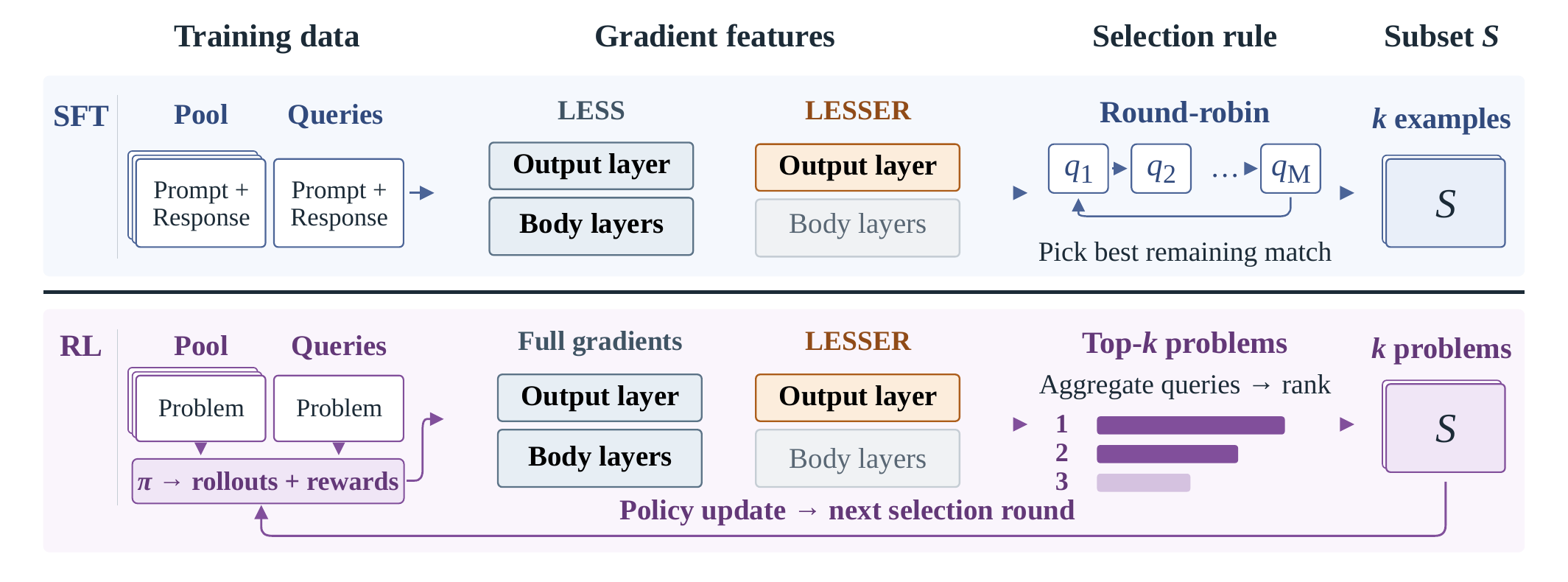}
\caption{\textbf{Data selection with output-layer gradient features in SFT and RL.}
Pool and query representations determine similarity scores, while the selection rule produces the training subset.
\LESSER{} replaces the gradient representation within this pipeline.}
\label{fig:selection-pipeline}
\end{figure}

The same basic pipeline applies across SFT and RL (Figure~\ref{fig:selection-pipeline}):
(1) compute features for the candidate pool and query set;
(2) compare candidate and query features to obtain relevance scores;
(3) apply a selection rule to these scores to choose a size-$k$ subset $S$; and
(4) train the model on the selected subset.
The precise selection rule can differ across settings.
In this work, we hold the selection procedure fixed within each setting and study the choice of feature representation.

\subsection{Gradient features}
\label{sec:background:gradients}

The feature-based view leaves open the central question of what information the representation $\phi_p$ should preserve.
A natural goal is to represent each candidate according to how training on it would affect performance on the target task.
We express this effect through a loss function, whose precise form depends on the post-training setting.

In SFT, each example $p=(x,y)$ consists of a prompt $x$ and target response $y$.
Let $\pi_\theta(y\mid x)$ denote the model's probability of the response under parameters $\theta$.
The supervised response loss is
\(
L_p(\theta) = -\log \pi_\theta(y\mid x).
\)
The same loss is defined for pool and query examples.
In RL, an example $p$ is instead a problem whose sampled responses and rewards determine a policy-optimization loss $L_p(\theta)$.
Following the online selection setup of GradAlign~\citep{yang2026gradalign}, this loss is constructed from rollouts under the current policy and changes as the policy is updated.
In either setting, each candidate and query is therefore associated with a scalar training objective, giving us a common way to describe the effect of training on that example.

Gradients provide a natural feature for these objectives: they describe the local parameter update induced by an example, and their alignment has a direct first-order interpretation in terms of how training on one example affects another.
For a candidate $p\in\mathcal D$ and query $q\in\mathcal Q$, define
\[
G_p = \nabla_\theta L_p(\theta), \qquad G_q = \nabla_\theta L_q(\theta).
\]
Consider taking one SGD step on candidate $p$ with learning rate $\eta>0$.
Then $\theta'=\theta-\eta G_p$, and a first-order Taylor expansion of the query loss gives
\[
L_q(\theta') - L_q(\theta)
\approx
\nabla_\theta L_q(\theta)^\top(\theta'-\theta)
=
-\eta G_q^\top G_p.
\]
Thus, when $G_q^\top G_p>0$, training on $p$ is predicted to reduce the loss on query $q$ to first order.

Standard gradient inner products then give a natural notion of similarity.
Because gradient magnitude can vary with factors such as sequence length, methods such as \LESS{}~\citep{xia2024less} further perform normalization or use cosine similarity instead.
Moreover, since $G_p\in\mathbb R^{|\theta|}$ has one coordinate per model parameter, gradient features are often projected to a lower-dimensional space for efficient storage and similarity computation~\citep{park2023trak,xia2024less}.
However, the main drawback of gradient features is computational cost: obtaining a full-model gradient requires backpropagating through the entire model for every scored example.

\section{\LESSER{}: feature extraction in one forward pass}
\label{sec:method}

Our main idea is to use output-layer gradients in place of full-model gradients, building on their use in settings such as active learning~\citep{ash2019deep} and data attribution~\citep{deng2025efficient,ran2026sketching}.
\Cref{fig:selection-pipeline} illustrates this substitution: \LESSER{} replaces the full-model gradients used by methods such as \LESS{} with output-layer gradients while leaving the selection rule unchanged.
Crucially, in our SFT and RL settings, these output-layer gradients can be computed from forward-pass quantities alone, avoiding backpropagation through the model.

A simple but key observation makes this possible: the standard SFT and RL scoring objectives we consider share a weighted token-level form:
\begin{equation}
\label{eq:weighted-token-loss}
L_p(\theta)
=
-\sum_t \alpha_t \log \pi_\theta(y_t\mid x,y_{<t}).
\end{equation}
We treat the per-token weights \(\alpha_t\) as fixed when differentiating.
In SFT, they select the supervised positions; in RL, they additionally incorporate response advantages and token weighting from the GradAlign scoring objective~\citep{yang2026gradalign}.
For this common form, the output-layer gradient can be written explicitly using forward-pass quantities.

\newcommand{\propoutputgradientstatement}{%
Let \(W\) be the final vocabulary projection, \(h_t\) the hidden state predicting \(y_t\), and \(e_{y_t}\) the one-hot vector for \(y_t\).
Let \(W_{\mathrm{out}}\) be a copy of \(W\) used only as the readout, so that \(h_t\) does not depend on it.
Then, the output-layer gradient of~\cref{eq:weighted-token-loss} satisfies
\[
g_p
\coloneqq
\nabla_{W_{\mathrm{out}}} L_p(\theta)
=
\sum_t
\alpha_t
\bigl(\operatorname{softmax}(Wh_t)-e_{y_t}\bigr)
h_t^\top.
\]
}
\begin{proposition}
\label{prop:output-gradient}
\propoutputgradientstatement
\end{proposition}

\Cref{prop:output-gradient} follows directly from the softmax cross-entropy gradient; Appendix~\ref{app:output-gradient-proof} gives a full proof.
Since both \(h_t\) and \(\operatorname{softmax}(Wh_t)\) are available from the forward pass, \(g_p\) can be computed without backpropagating through the model.
We use \(\operatorname{vec}(g_p)\) as the feature for each candidate and query in place of the full-model gradient \(G_p=\nabla_\theta L_p(\theta)\).

When needed for efficient storage and similarity computation, we project
\(\operatorname{vec}(g_p)\) to a lower dimension following standard practice for gradient features~\citep{johnson1984extensions,park2023trak};
Appendix~\ref{app:sketch} specifies our two-sided projection.

\section{\LESSER{} is effective at targeted data selection}
\label{sec:experiments}

\begin{figure*}[t]
\centering
\includegraphics[width=\linewidth]{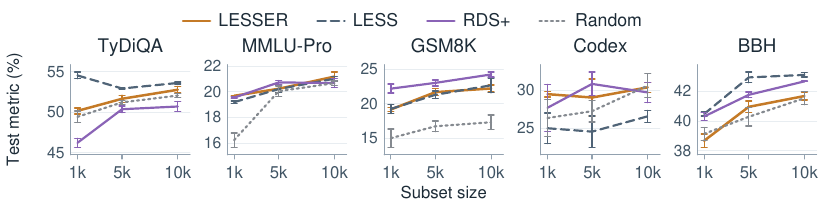}
\vspace{-5pt}
\caption{
\textbf{\LESSER{} closely tracks \LESS{} on downstream performance.}
Panels show Llama-2-7B; values are means over three training seeds, and error bars show one sample standard deviation. Across all $60$ model--task--budget cells, the absolute score gap of \LESSER{} averages $1.3$ points against \LESS{}, $2.2$ against \RDS{}, and $2.6$ against Random.
}
\label{fig:sft-budget-curves}
\vspace{-5pt}
\end{figure*}

We evaluate \LESSER{} across targeted SFT, RL, and distillation teacher selection.
We compare its effectiveness with the corresponding full-gradient methods and measure the resulting computational savings.

\paragraph{SFT.}
\label{sec:rq1-sft}

We take $197{,}196$ Tulu~V2~\citep{ivison2023camels} instruction--response samples as our candidate pool, and take TyDiQA, MMLU-Pro, GSM8K, Codex, and BBH as our query and test sets, following prior work~\citep{xia2024less,ivison2025large,nayak2026critical}.
For data-selection methods, we use \textbf{\LESSER{}}, \textbf{\LESS{}}, \textbf{Random} (wherein samples are randomly chosen), and \textbf{\RDS{}} (hidden-state based), where selection is done with a standard round-robin rule until $k$ samples are chosen.
For selection sizes $k\in\{1{,}000,5{,}000,10{,}000\}$, we fully fine-tune Llama-2-7B, Llama-3.2-3B, Qwen3-4B-Base, and OLMo3-7B on each selected subset with three seeds.
We show Llama-2-7B in the main text and report the full model--task--budget grid in Appendix~\ref{app:full-grid}.
Baseline results are from \citet{nayak2026critical} except for MMLU-Pro, where we rerun all methods under our evaluation setup. 
Another caveat is that \LESSER{} uses only the base model, while \LESS{} uses Adam-preconditioned LoRA~\citep{hu2021lora} gradients at four warmup checkpoints for computational tractability and empirical improvements.
The two pipelines therefore differ beyond the feature, but \LESSER{} is much cheaper and more convenient to extract.
Appendix~\ref{app:sft-protocol} gives the full training protocol and replication checks.

\researchfinding{rf:sft}{
\LESSER{} is competitive with \LESS{} on downstream SFT performance.
}
On Llama-2-7B, the two methods are nearly identical on GSM8K and MMLU-Pro, while \LESS{} is stronger on TyDiQA and BBH and \LESSER{} is stronger on Codex (\cref{fig:sft-budget-curves}).
Across the four models, \LESSER{} frequently improves over Random, while \RDS{} is also competitive on several benchmarks.
We next ask whether \LESSER{}'s similarity scores behave consistently across the pool: do more similar examples lead to lower query loss?

\begin{figure*}[t]
\centering
\includegraphics[width=\linewidth]{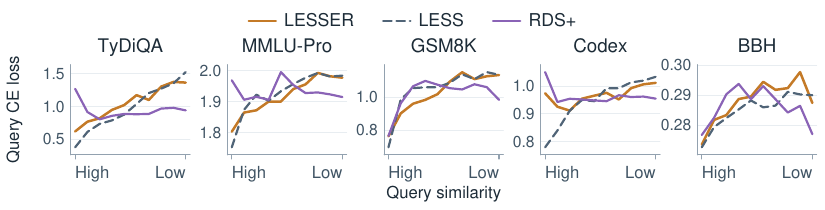}
\vspace{-5pt}
\caption{
\textbf{Higher-similarity training examples lead to lower query loss.}
Query loss after SFT on Llama-2-7B, stratified by example score. For both \LESS{} and \LESSER{}, examples with higher query similarity generally yield lower query loss after fine-tuning.
}
\label{fig:bin-loss-curves}
\vspace{-10pt}
\end{figure*}

A desirable property of data selectors is that they should rank data by their effect on the query loss: high-ranked samples should reduce it the most, while low-ranked samples the least.
To test that this holds for \LESSER{}, we use a similarity-bin diagnostic~\citep{nayak2026critical} where we partition the candidate pool into ten deciles.
We then fine-tune a separate model on the top $500$ samples from each bin and measure its loss on the query set.
If the similarity score captures target relevance, query loss should increase as we move from higher- to lower-similarity bins, giving a positive Spearman correlation between bin index and query loss.

\researchfinding{rf:bins}{
Higher \LESSER{} similarity predicts lower query loss in SFT.
} 
On Llama-2-7B, \LESSER{} similarity strongly predicts lower query loss, with a mean Spearman correlation of $0.86$, close to \LESS{} and far above \RDS{} (\cref{fig:bin-loss-curves}).
This relationship generalizes across models: \LESSER{} is positively correlated with query loss decrease in $19$ of $20$ model--task pairs, though somewhat less strongly than \LESS{} overall.

\paragraph{RL.}
\label{sec:rq1-rl}

We evaluate \LESSER{} in two GradAlign settings: avoiding problems with corrupted rewards and selecting problems during online RL training. In both settings, we replace only the full-gradient representation with output-layer gradients while keeping the selection procedure fixed.
We first test whether \LESSER{} can avoid problems whose rewards carry no useful learning signal.
Following \citet{yang2026gradalign}, we corrupt the rewards of half the candidate problems with coin flips and select a subset against AMC22 queries.

\researchfinding{rf:rl-corrupted}{\LESSER{} closely matches full policy gradients in avoiding problems with corrupted rewards.}
\LESSER{} selects $29.7\%$ corrupted problems, close to $27.3\%$ for full gradients and well below \textbf{Random} (\cref{fig:rl-noisy-selection}).
\textbf{AccGreedy}, which selects problems whose empirical accuracy is closest to $50\%$, instead favors corrupted problems because their coin-flip rewards push empirical accuracy toward $50\%$.
Appendix~\ref{app:rl-corrupted-protocol} gives the full protocol.

Another setting is in online RL, where in our case data selection happens after every 10 policy gradient updates.
We target Countdown using a mixed-domain training pool in which only $11.8\%$ of the problems are from Countdown, with the remainder from WebInstruct~\citep{yue2024mammoth2} and DAPO~\citep{yu2026dapo}.
Across five selection rounds, we compare \LESSER{} with full-gradient GradAlign, \textbf{dynamic sampling}~\citep{yu2026dapo}, which retains problems with nonzero rollout advantage, and \textbf{Random}, while training Qwen2.5-Math-1.5B-Instruct with GRPO~\citep{shao2024deepseekmath}.

\researchfinding{rf:rl-domain}{
\LESSER{} closely matches full policy gradients in targeted online RL training.
}
\LESSER{} closely tracks full-gradient GradAlign, and both reach the same final held-out accuracy (\cref{fig:rl-training}).
\LESSER{} reaches Random's final accuracy after only $15$ steps, using \textbf{$70\%$ fewer GRPO steps}.

\begin{figure}[tbp]
\centering
\begin{minipage}[t]{0.48\linewidth}
\vspace{0pt}
\centering
\includegraphics[width=\linewidth]{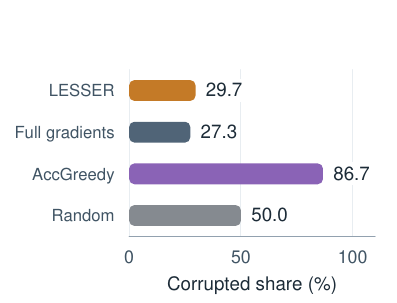}\par
\caption{\textbf{\LESSER{} closely matches full gradients in detecting problems with corrupted rewards.}
Corruption share among the selected problems. \textbf{Lower} is better.
}
\label{fig:rl-noisy-selection}
\end{minipage}\hfill
\begin{minipage}[t]{0.48\linewidth}
\vspace{0pt}
\centering

\includegraphics[width=\linewidth]{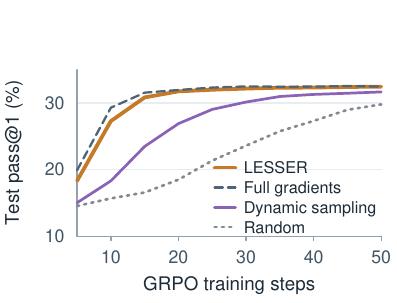}\par
\caption{\textbf{\LESSER{} closely tracks full gradients in RL.}
Held-out Countdown accuracy over GRPO training steps, averaged over two seeds. \textbf{Higher} is better.
}\label{fig:rl-training}
\end{minipage}
\end{figure}

\vspace{-10pt}

\paragraph{Distillation teacher selection.}
\label{sec:grace}
GRACE performs distillation teacher selection by having the student model backpropagate on each candidate teacher's generated responses; it then uses these gradients to rank and select teachers.
As in SFT and RL, \LESSER{} instead replaces full-gradient features with output-layer gradients, while leaving the rest of the GRACE pipeline unchanged.

\researchfinding{rf:grace}{
\LESSER{} closely preserves full-gradient GRACE teacher rankings.
}
Across all settings, the teacher rankings remain closely aligned, with Spearman correlations of $0.94$--$0.98$ (Table~\ref{tab:current-grace}), and \LESSER{} recovers the same top-ranked teacher as full-gradient GRACE.
That is, even though output-layer gradients are used, there is no drop in the quality of teacher selection.

\paragraph{Cost analysis.}
\label{sec:cost}

We give detailed accounting in Appendix~\ref{app:sft-pipeline-cost}, with highlights in Table~\ref{tab:selection-cost} and below.

\paragraph{SFT.}
\LESSER{} extracts features with a single forward pass, whereas the \LESS{} setup used in our experiments extracts LoRA-gradient features with full backpropagation at four checkpoints.
Counting FLOPs using \LESS{}'s configuration, this reduces feature-extraction FLOPs by $9.7\times$ and, extrapolating our A100 measurements to the full candidate pool, wall-clock extraction time by \textbf{$17.8\times$}.
Even against a single \LESS{} checkpoint, extraction is $4.4\times$ faster in wall clock.
Feature extraction accounts for $97\%$ to $99.7\%$ of \LESS{}'s combined extraction and fine-tuning FLOPs across our tested budgets, making it the dominant computational cost.

\paragraph{RL.}
Full-gradient GradAlign requires a separate forward--backward scoring pass after rollout generation, whereas \LESSER{} scores with projected output-layer features, which need only a forward pass. 
This reduces scoring FLOPs by $3.0\times$.
On the responses of one selection round, the added scoring time falls from $63.3$ to $18.1$ minutes on one A100, a \textbf{$3.5\times$} reduction.

\paragraph{Distillation.}
Full-gradient GRACE backpropagates through the student for every teacher response, whereas \LESSER{} requires only a forward pass.
Across the three teacher-selection settings, this reduces feature-extraction time by $3.9$ to $5.4\times$ while also lowering peak GPU memory.

\section{Why are output-layer gradients effective?}
\label{sec:analysis-new}

Since \LESSER{} often matches full-gradient methods downstream, we ask why it can do so using just the output-layer gradient.
In GRACE the two methods rank teachers almost identically and in RL they largely agree, but in SFT their selections differ.
We therefore ask whether \LESSER{} and \LESS{} might still update the model in similar directions.
At the base model where selection happens, we find that batches selected by \LESSER{} and \LESS{} produce aligned full-model training gradients, much more so than \RDS{} or Random.
In addition, both \LESSER{} and \LESS{} achieve most of their query-loss decrease within the first ten updates.

\paragraph{\LESSER{} can select different data from \LESS{}.}
\label{sec:analysis-per-example}

In SFT, we test whether the similar downstream performance of \LESSER{} and \LESS{} can be explained by choosing similar data.
We check for overlap using Jaccard similarity, $\lvert S\cap S'\rvert/\lvert S\cup S'\rvert$, and we compare rankings using Spearman rank correlation.

At $k=5{,}000$ on Llama-2-7B, the selected sets have mean Jaccard similarity $0.037$ (Appendix~\ref{app:selection-overlap}). 
In an audit across the $20$ model--task pairs, output-layer and full-gradient rankings have median Spearman correlation $0.52$ (Appendix~\ref{app:local-utility}).
This level of disagreement is specific to SFT: in RL, the two selections share $69\%$ of their problems in the first round, while in GRACE the statistics used for teacher selection remain highly correlated (Appendix~\ref{app:rl-grace}).
This leaves the SFT result unexplained, so we turn to how the selected examples affect training and ask whether the different selections induce similar updates.

\begin{table}[t]
\centering
\setlength{\belowcaptionskip}{4pt}
\caption{\textbf{Output-layer gradients are much cheaper to extract.}
Cost analysis on one A100~80GB for SFT feature extraction and RL scoring.
}
\label{tab:selection-cost}
\small
\renewcommand{\arraystretch}{1.12}
\setlength{\tabcolsep}{4pt}

\begin{minipage}[t]{0.36\linewidth}
\vspace{0pt}
\centering
\sublabel{a}{tab:sft-extraction-cost}
\begin{tabular*}{\linewidth}{@{\extracolsep{\fill}}l|r|r@{}}
\textbf{SFT Selector} & Time (h) & Speedup \\
\midrule
\LESSER{} & $3.16$ & $\mathbf{17.8}\times$ \\
\LESS{} & $56.16$ & $1.0\times$ \\
\bottomrule
\end{tabular*}
\end{minipage}\hfill
\begin{minipage}[t]{0.612\linewidth}
\vspace{0pt}
\centering
\sublabel{b}{tab:rl-round-time}
\begin{tabular*}{\linewidth}{@{\extracolsep{\fill}}l|r|r|r@{}}
\textbf{RL Selector} & Added time (min) & Speedup & Feature size \\
\midrule
\LESSER{} & $18.1$ & $\mathbf{3.5}\times$ & $32.8$~kB \\
Full gradients & $63.3$ & $1.0\times$ & $6.2$~GB \\
\bottomrule
\end{tabular*}
\end{minipage}
\end{table}

\paragraph{The selected batches align.}
\label{sec:analysis-batches}

We first characterize how a selected batch affects training.
For a batch $S$, let $G_S$ denote its full-model gradient at the base model $\theta_0$.
Our goal is to compare the training directions induced by two selections, which requires us to exclude gradient components that are common across the candidate pool.
In particular, even independently sampled random batches can have high cosine similarity because their gradients share the population mean.
We observe this effect empirically in Appendix~\ref{app:gradient-similarity}.
We therefore subtract the expected gradient $\mu$ of a random batch and compare the centered gradients $G_S-\mu$ and $G_{S'}-\mu$ using the centered batch-gradient alignment defined in Appendix~\ref{app:gradient-similarity}.
This measures whether two selections emphasize similar training directions beyond random sampling.

We evaluate gradient alignment using the actual training batch size of $128$ across $20$ model--task pairs (Figure~\ref{fig:alignment-base}).
As a reference for the scale of this measurement, we first compare different \LESS{} batches drawn from the same $k=5{,}000$ selection; their median alignment is $0.53$.
Against this reference, \LESSER{} batches show substantial alignment with \LESS{}, with a median alignment of $0.29$.
In contrast, the median alignment with \LESS{} is $0.03$ for \RDS{} and approximately zero for Random.
Thus, under the batch size used in training, \LESSER{} batches align with \LESS{} batches far more than \RDS{} or random batches do.

\paragraph{The effect of selection happens early in training.}
\label{sec:analysis-first-updates}

Above, we showed that \LESSER{} and \LESS{} select batches whose full-model gradients are aligned at the base model, despite selecting largely different examples.
The remaining question is when selection affects the subsequent training run.
We therefore track query loss throughout the full $80$-update SFT runs.

The query-loss gains from selection are concentrated early in training.
By step $10$, \LESSER{} and \LESS{} have achieved a median $84\%$ and $82\%$, respectively, of their maximum query-loss decrease over the full $80$ updates, compared with $35\%$ for Random (Figure~\ref{fig:training-directions}).
This connects the first-order selection at the base model to the reduction in query loss.

\section{Discussion and Conclusion}
\label{sec:conclusion}

\paragraph{Discussion.}
Our results suggest three broader lessons. 

First, compared to hidden-state features, output-layer gradients preserve more of the signal used by full-gradient data-selection methods at a similar extraction cost. 
They are also more directly tied to the optimization objective, making them a natural starting point when extraction cost matters. 

Second, agreement is strongest at the level each pipeline acts on. SFT trains on batch-aggregated gradients, RL selects problems after aggregating sampled responses, and GRACE ranks teachers using statistics over many responses. This suggests a broader principle: efficient selection features need not recover per-example full gradients, as long as they preserve the aggregate geometry that each pipeline uses.

Third, full-gradient fidelity may therefore be the wrong objective for evaluating a selection proxy. Prior work on training-data attribution~\citep{grosse2023studying} shows that recovering full-network influence can require information distributed across layers. This suggests that fine-grained approximation may be important for influence estimation or data attribution, while selection from realistic, already-curated LLM data pools may only require preserving which examples or subsets are useful.

\begin{figure}[t]
\centering
\begin{minipage}[t]{0.49\linewidth}
\centering
\includegraphics[width=\linewidth]{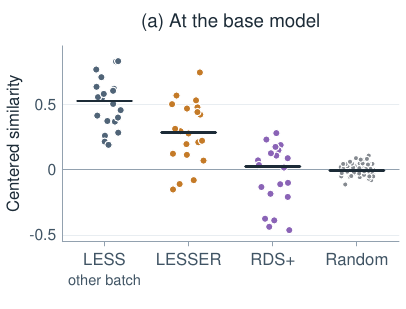}
\end{minipage}\hfill
\begin{minipage}[t]{0.49\linewidth}
\centering
\includegraphics[width=\linewidth]{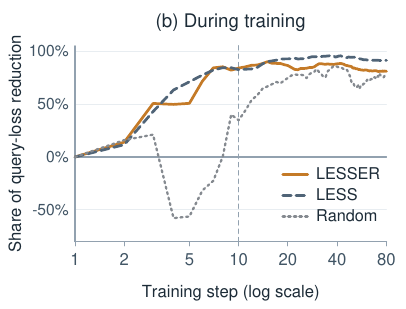}
\end{minipage}
\caption{
\textbf{\LESSER{}-selected batches induce gradients aligned with those from \LESS{}.}
(a) Centered batch-gradient alignment between \LESS{} batches and batches selected by each method, at the base model.
Dots denote model--task pairs, and bars show medians.
(b) Median share across model--task pairs of the query-loss decrease reached by each training step.
}
\label{fig:gradient-similarity}
\figsublabel{a}{fig:alignment-base}
\figsublabel{b}{fig:training-directions}
\end{figure}

\paragraph{Limitations.}
Our experiments span multiple models, tasks, and selection settings, but do not cover all post-training regimes or selection objectives.
We also focus primarily on near-initialization training effects, leaving a fuller understanding of how output-layer gradient signals evolve over longer training trajectories to future work.
Our SFT comparison is between complete pipelines rather than a matched-checkpoint ablation, so \LESSER{} and \LESS{} differ in more than just their features. 
Nevertheless, \LESSER{} needs neither warmup training nor backpropagation, and it remains competitive with this more expensive pipeline.

\paragraph{Future Work.}
Future work could develop diagnostics that predict, before running selection, when output-layer gradients will recover the rankings or update directions induced by full gradients.
A second direction is to identify which computations inside the model make this agreement possible, for example by tracing how selection-relevant information is transformed across layers.
Finally, a stronger theory could characterize when such agreement should emerge under realistic, structured gradient distributions.

\paragraph{Conclusion.}
We study whether output-layer gradients can replace full-gradient features for data selection in LLM post-training.
Across SFT, RL, and distillation teacher selection, they often match full-gradient methods downstream at substantially lower cost.
In RL and GRACE, the resulting selection statistics largely agree, while in SFT, different rankings nevertheless induce aligned batch updates.
Output-layer gradients therefore offer a simple, efficient alternative to full gradients.

\bibliography{iclr2027_conference}
\bibliographystyle{iclr2027_conference}

\clearpage
\appendix
\paragraph{Appendix guide.}

Appendix~\ref{app:setup} gives the experimental protocols and cost accounting; Appendix~\ref{app:derivations} proves Proposition~\ref{prop:output-gradient} and gives the projection of the output-layer gradient. The results follow the main experiments: SFT in Appendix~\ref{app:sft-results} and GRACE teacher selection in Appendix~\ref{app:grace-controls}. Appendix~\ref{app:utility-results} reports the overlap of the deployed selections, and Appendix~\ref{app:new-finite} supports Section~\ref{sec:analysis-new}.

\section{Experimental Setup Details}
\label{app:setup}

This appendix specifies SFT on selected examples, RL on selected problems, GRACE teacher ranking, and cost accounting.

\subsection{SFT}
\label{app:sft-protocol}
\paragraph{Baselines and MMLU-Pro.}
The SFT grid uses the candidate pool, target tasks, training recipe, and evaluation harness of \citet{nayak2026critical}. On TyDiQA, GSM8K, Codex, and BBH, we add \LESSER{} runs to their released Random, \RDS{}, and \LESS{} results. Under the zero-shot chat-template evaluation of MMLU-Pro, many fine-tuned models do not produce an answer in the expected format, so accuracy depends mainly on the answer format. We therefore evaluate MMLU-Pro five-shot with the \texttt{mmlu\_pro} task of the LM Evaluation Harness on all $12{,}032$ test questions, and we retrain all four methods, using the released selections for the three baselines, with three seeds for every model and budget. All MMLU-Pro results in the paper use this evaluation.

\paragraph{Candidate pool.}
The pool is the subsampled Tulu~V2 instruction mixture of $197{,}196$ instruction--response examples
\citep{ivison2023camels}, in its default order. The same pool is used for every (model, task, method, budget)
cell, so methods differ only in which subset they select.

\paragraph{Target tasks.}
We use the five target tasks of \citet{nayak2026critical}, listed in Table~\ref{tab:app-tasks}. Each task is specified by a small in-distribution validation/query set, used both to define the target and to rank pool examples for selection, and is scored on a held-out test split. 

\begin{table}[t]
\centering
\caption{Target tasks. ``Val.\ queries'' is the query-set size used to specify each target; metrics and splits follow \citet{nayak2026critical}, except that MMLU-Pro is evaluated five-shot.}
\label{tab:app-tasks}
\small
\begin{tabular}{l|l|c|l}
Task & Domain & Val.\ queries & Metric \\
\midrule
TyDiQA \citep{clark2020tydi}            & multilingual QA   & $9$  & F1 \\
MMLU-Pro \citep{wang2024mmlu}           & knowledge / reasoning & $70$ & accuracy \\
GSM8K \citep{cobbe2021training}         & grade-school math word problems       & $8$  & exact match \\
Codex \citep{chen2021evaluating}    & Python code generation     & $16$ & pass@$10$ \\
BBH \citep{suzgun2023challenging}       & challenging BIG-Bench reasoning       & $81$ & exact match \\
\bottomrule
\end{tabular}
\end{table}

\paragraph{Models.}
We evaluate four base models under full fine-tuning: Llama-3.2-3B~\citep{grattafiori2024llama}, Llama-2-7B~\citep{touvron2023llama}, Qwen3-4B-Base~\citep{yang2025qwen3}, and OLMo3-7B~\citep{olmo2025olmo}.

\paragraph{Training.}
Selection is followed by full fine-tuning: $2$ epochs, learning rate $2\times10^{-5}$ with a linear schedule and $3\%$ warmup, weight decay $0$, effective batch size $128$ from per-device batch $1$ and gradient accumulation $128$, bf16, and maximum sequence length $2048$.

\paragraph{Evaluation.}
We run inference with fixed prompts and decoding, and score each task with the metric listed in Table~\ref{tab:app-tasks}.
Codex uses HumanEval~\citep{chen2021evaluating} with HumanEvalPack instructions~\citep{muennighoff2024octopack}. 

\paragraph{Selection rule and budgets.}
\LESSER{}, \RDS{}, and \LESS{} share the round-robin rule. Budgets are $k\in\{1{,}000,5{,}000,10{,}000\}$, and every model--task--budget comparison reports three training seeds. Appendix~\ref{app:full-grid} gives the means and sample standard deviations.

\paragraph{Replication check.}
\label{app:baseline-replication}

We check the training and evaluation pipeline by rerunning released selections in our pipeline. For Llama-3.2-3B at $k=1{,}000$, we rerun the released \LESS{} selections and Random with three seeds and compare with Table~\ref{tab:main-results-k1000}; for Qwen3-4B-Base and OLMo3-7B, we rerun the released \LESS{} TyDiQA selection at seed $0$ (Table~\ref{tab:baseline-replication}). The \LESS{} reruns agree within $0.9$ points. Our Random reruns draw new random subsets instead of the released ones and differ by up to $3.0$ points. MMLU-Pro is omitted because we rerun it five-shot for all methods.

\begin{table}[htbp]
\centering
\caption{Pipeline replication, reported value / our rerun.}
\label{tab:baseline-replication}
\small\begin{tabular}{l|l|r|r}
Model & Task & \LESS{} & Random \\
\midrule
Llama-3.2-3B & TyDiQA & $66.4$ / $66.1$ & $53.7$ / $50.7$ \\
Llama-3.2-3B & GSM8K & $34.2$ / $34.0$ & $29.2$ / $27.6$ \\
Llama-3.2-3B & Codex & $56.3$ / $55.4$ & $59.2$ / $59.9$ \\
Llama-3.2-3B & BBH & $47.6$ / $47.2$ & $47.1$ / $47.1$ \\
\midrule
Qwen3-4B-Base & TyDiQA & $71.6$ / $71.4$ & -- \\
OLMo3-7B & TyDiQA & $62.5$ / $62.5$ & -- \\
\bottomrule
\end{tabular}
\end{table}

\paragraph{Feature extraction.}
\LESSER{} features are extracted from the pretrained base model, unlike \LESS{}. We sum residual--hidden-state outer products over supervised response positions, apply the $64\times128$ two-sided projection in Appendix~\ref{app:sketch}, then flatten and normalize the resulting $8{,}192$ entries. Pool and query examples share projection matrices. Prompt tokens are masked from the response loss.

\paragraph{Similarity decile training.}
We rank the pool by each method's similarity to the queries, split it into deciles with bin $0$ closest to the queries, and take the top $500$ examples from each bin. For each bin we fine-tune a separate model with the downstream recipe above and one seed, and measure the mean loss on the query set; Appendix~\ref{app:bin-loss} reports the results.

\subsection{Targeted reinforcement learning}
\label{app:rl-protocol}

\label{app:rl} 

We follow the setup and released implementation of GradAlign~\citep{yang2026gradalign}, and change only the selection feature. GradAlign scores a problem $p$ by the gradient $G_p$ of its advantage-weighted response loss $L_p(\theta)=\sum_jA_j\ell_j(\theta)$, where the $K$ sampled responses $y_j$ have rewards $R_j$, centered advantages $A_j=R_j-\bar R$, and token-weighted response losses $\ell_j$. We replace $G_p$ with the output-layer gradient of the same loss,
\begin{equation}
g_p=\sum_{j,t}w_{jt}A_j\,r_{jt}h_{jt}^{\top},
\label{eq:rl-output-gradient}
\end{equation}
where $w_{jt}$ are the fixed token weights of the released scorer, $h_{jt}$ is the final hidden state that predicts token $y_{jt}$, and $r_{jt}$ is its prediction residual. Query problems are aggregated with the same weights, and candidates are ranked by cosine similarity with the aggregated query feature. A problem whose sampled rewards are all equal has zero advantages and therefore scores zero under both features.

\subsubsection{Data, training, and evaluation}

The policy is Qwen2.5-Math-1.5B-Instruct~\citep{yang2024qwen2}. Following the data proportions of GradAlign Scenario~2, we construct a pool of $34{,}000$ problems: $4{,}000$ Countdown problems, $20{,}000$ WebInstruct problems~\citep{yue2024mammoth2}, and $10{,}000$ DAPO problems~\citep{yu2026dapo}. The Countdown instances are drawn separately from those used in GradAlign. We additionally use $200$ Countdown problems as queries and $500$ held-out Countdown problems for evaluation; both sets are disjoint from the training pool. Table~\ref{tab:rl-protocol} summarizes the experimental settings.

Training proceeds for multiple selection rounds. 
In each round, a method scores a chunk of $5{,}120$ candidate problems from the pool, selects $256$ of them, trains on them for $10$ GRPO steps, and then generates fresh responses for both the candidate and query problems using its updated policy. 
Random instead samples $256$ problems uniformly from the whole pool in each round.
These new responses are used for selection in the next round. All methods start from the same initial policy and use the same released reward functions and GRPO hyperparameters.

For evaluation, we sample $16$ responses for each of the $500$ held-out Countdown problems and report their mean correctness. Results are averaged over seeds $42$ and $43$ at evaluation steps completed by all runs.

\begin{table}[htbp]
\centering
\caption{RL selection, training, and evaluation settings.}
\label{tab:rl-protocol}
\small\begin{tabular}{@{}p{0.36\linewidth}|p{0.59\linewidth}@{}}
Setting & Value \\
\midrule
Candidates and selection & $5{,}120$ candidate problems; $256$ selected per round \\
Scoring responses & $16$ per candidate and per query problem \\
Training between selections & $10$ GRPO steps \\
Training batch and responses & $128$ problems per step; $128$ responses per problem \\
Optimizer & AdamW \citep{loshchilov2017decoupled}, learning rate $10^{-6}$ \\
Sampling & Temperature $0.7$; maximum response length $3{,}072$ tokens \\
KL penalty & Disabled \\
Evaluation & $500$ test problems; $16$ responses per problem \\
Training repetitions & Two seeds per method, averaged, for all four methods \\
\bottomrule
\end{tabular}
\end{table}

\subsubsection{Selection under corrupted rewards}
\label{app:rl-corrupted-protocol}
Following the corrupted-reward setting of GradAlign~\citep{yang2026gradalign}, we use Qwen3-8B-Base~\citep{yang2025qwen3}, with $512$ DAPO candidate problems~\citep{yu2026dapo} and $43$ AMC22 query problems. We sample $64$ responses per problem. For half of the candidates, we replace the correctness rewards of their sampled responses with independent fair-coin flips; the remaining candidates retain their true rewards. At the base model, each method selects $128$ candidates. AccGreedy ranks candidates by how close their empirical reward rate is to $0.5$. \LESSER{}, full gradients, and AccGreedy select $38$, $35$, and $111$ corrupted candidates, respectively, compared with $64$ expected under uniform random selection.

\subsection{GRACE teacher ranking}
\label{app:grace-protocol}

We use the released GRACE scorer~\citep{panigrahi2025good} with its default settings, replacing only the student's gradient features with \LESSER{}. The teachers, prompts, responses, $512$-dimensional random projection, and scoring rule are otherwise unchanged. GRACE scores four responses from each teacher to each of $512$ shared prompts, generated with vLLM~\citep{kwon2023efficient}; for GSM8K, these are four of $16$ sampled responses per prompt, selected by the released data preparation.

We evaluate three settings: (1) GSM8K with a Llama-3.2-1B student and $14$ teachers, consisting of the original ten teachers plus Qwen2.5-14B-Instruct~\citep{qwen2025qwen25technicalreport}, Phi-4~\citep{abdin2024phi}, OLMo-2-13B-Instruct~\citep{olmo20242}, and Gemma-2-27B-it~\citep{team2024gemma}; (2) GSM8K with an OLMo-2-1B student~\citep{olmo20242} and $10$ teachers; and (3) MATH~\citep{hendrycks2021measuring} with a Llama-3.2-3B student and $10$ teachers. Appendix~\ref{app:grace-controls} reports the resulting teacher rankings.

\subsection{Efficiency accounting and timing}
\label{app:efficiency}

\paragraph{SFT extraction FLOPs.}
We compare feature-extraction FLOPs for \LESSER{} and the released four-checkpoint \LESS{} configuration on Llama-2-7B.
The quantities used in the calculation are summarized in \cref{tab:sft-flop-params}.

\begin{table}[h]
\centering
\small
\caption{Quantities used for the Llama-2-7B SFT feature-extraction FLOP calculation.}
\label{tab:sft-flop-params}
\begin{tabular}{lcr}
\toprule
Quantity & Symbol & Value \\
\midrule
Model parameters & $P$ & $6{,}738{,}415{,}616$ \\
LoRA parameters & $P_{\rm LoRA}$ & $134{,}217{,}728$ \\
Pool examples & $N$ & $197{,}196$ \\
Processed tokens & $T$ & $86{,}343{,}509$ \\
Response tokens & $T_r$ & $25{,}430{,}302$ \\
Vocabulary size & $|V|$ & $32{,}000$ \\
Hidden dimension & $d$ & $4{,}096$ \\
\LESS{} projection dimension & $D$ & $8{,}192$ \\
\LESSER{} projection dimensions & $(m,n)$ & $(64,128)$ \\
\LESS{} checkpoints & $C$ & $4$ \\
\bottomrule
\end{tabular}
\end{table}

A forward pass costs approximately $2PT$ FLOPs~\citep{kaplan2020scaling,hoffmann2022training}.
\LESSER{} requires one forward pass, while its two-sided projection of the output-layer gradient adds
$
2T_r\bigl(|V|m+dn+mn\bigr),
$
giving
$$
F_{\LESSER{}}
=
2PT + 2T_r\bigl(|V|m+dn+mn\bigr).
$$

At each checkpoint, \LESS{} performs a forward pass, propagates activation gradients through the frozen model, computes the LoRA-gradient terms, and projects each example's gradient to $D=8{,}192$ dimensions.
Its total cost is therefore
\[
F_{\LESS{}}
=
4\left(
4PT
+
6P_{\rm LoRA}T
+
2DP_{\rm LoRA}N
\right).
\]
Using the values in \cref{tab:sft-flop-params}, the resulting ratio is
\[
\frac{F_{\LESS{}}}{F_{\LESSER{}}}
=
9.7\times.
\]
Under the standard approximation of $2PT$ for a forward pass and $6PT$ for a forward--backward pass, abstracting away LoRA and projection overhead, four-checkpoint \LESS{} costs approximately $24PT$, giving the commonly used estimate of approximately $12\times$.

\paragraph{SFT extraction wall clock.}
\label{app:sft-pipeline-cost}

We measure SFT feature-extraction time for \LESSER{} and \LESS{} on the same A100~80GB and the same $10{,}000$-example prefix of the candidate pool, timing the full extraction pipeline from data loading through feature writing.
\LESSER{} takes $576$ seconds, while one checkpoint of released \LESS{} takes $2{,}563$ seconds.
Extrapolating both measurements to the full $197{,}196$-example pool and accounting for the four checkpoints used by \LESS{} gives $3.16$ hours for \LESSER{} and $56.2$ hours for \LESS{}, a $17.8\times$ difference.

\paragraph{SFT extraction FLOPs relative to downstream training.}
We next ask whether the extraction savings remain substantial once downstream fine-tuning is included.
For one target and training run, we model the combined extraction and fine-tuning cost as
$2PT_{\rm pool}+6PT_{\rm train}$ for \LESSER{} and
$F_{\rm LESS}+6PT_{\rm train}$ for \LESS{}, where $F_{\rm LESS}$ is the detailed four-checkpoint extraction cost above.
As shown in \cref{tab:sft-extraction-share}, feature extraction accounts for at least $97.0\%$ of \LESS{}'s modeled extraction-plus-training FLOPs across the tested budgets.
Consequently, the combined extraction and fine-tuning cost is $7.7$--$9.5\times$ larger for \LESS{} than for \LESSER{}. We assume two training epochs and equal mean token length for pool and selected examples.
These estimates exclude warmup training, query processing, selection, and evaluation.

\begin{table}[htbp]
\centering
\caption{
SFT computation costs for one target and training run.
}
\label{tab:sft-extraction-share}
\small
\begin{tabular*}{\linewidth}{@{\extracolsep{\fill}}r|r|r|r@{}}
 & \multicolumn{2}{c|}{Extraction share / \%} & \\
Budget $k$ & \LESSER{} & Four-checkpoint \LESS{} &
\shortstack{Extraction + training\\FLOP ratio} \\
\midrule
$1{,}000$  & $97.0$ & $99.7$ & $9.47\times$ \\
$5{,}000$  & $86.8$ & $98.5$ & $8.58\times$ \\
$10{,}000$ & $76.7$ & $97.0$ & $7.69\times$ \\
\bottomrule
\end{tabular*}
\end{table}

\paragraph{RL scoring FLOPs and storage.}
In its scoring pass, \LESSER{} runs one forward pass over each scored token and projects the output residual and hidden state before accumulating their outer product.
For vocabulary size $|V|=151{,}936$, hidden dimension $d=1{,}536$, and projection dimensions $m=64$ and $n=128$, the projection costs
\[
2(|V|m+dn)+2mn
=
19{,}857{,}408
\]
FLOPs per scored token, and the forward pass $2P=3.09\times10^9$ for $P=1{,}543{,}714{,}304$ parameters, $3.11\times10^9$ in total.
Full-gradient scoring instead requires a separate forward--backward pass, costing approximately
$
6P = 9.26\times10^9
$
FLOPs per scored token, $3.0\times$ the cost of \LESSER{}.

\begin{table}[htbp]
\centering
\caption{
Computation and per-problem storage for RL scoring.
}
\label{tab:rl-feature-cost}
\small
\begin{tabular}{l|r|r}
Feature & FLOPs / scored token & Stored feature \\
\midrule
\LESSER{} & $3.11\times10^9$ & $32.8$~kB \\
Full gradients & $9.26\times10^9$ & $6.2$~GB \\
\bottomrule
\end{tabular}
\end{table}

\paragraph{RL scoring wall clock.}
\label{app:rl-rollout-reuse}

\LESSER{} scores each selection round with projected output-layer features, computed in a separate forward pass over the sampled responses after rollout generation. On the $81{,}920$ candidate and $3{,}200$ query responses of one round, on one A100~80GB PCIe and in one job, this pass takes $18.1$ minutes, against $63.3$ minutes for the full-gradient forward--backward pass, a $3.5\times$ reduction.

\paragraph{GRACE extraction wall clock.}
\label{app:grace-timing}
We estimate the cost of extracting GRACE features for all teacher responses by timing a subset and extrapolating to the full dataset.
For each teacher, we randomly sample $32$ responses and extract their features one at a time on one A100~80GB, after three untimed warm-up responses.
We use the fp32 student and the same $512$-dimensional projector configuration for both methods.
For each teacher, we multiply the measured mean time per response by its $2{,}048$ responses, then sum across teachers to obtain the total extraction times in Table~\ref{tab:grace-timing}.
Peak memory is measured directly during these runs and is not extrapolated.
Under this estimate, \LESSER{} is $3.9$--$5.4\times$ faster than full-gradient extraction while using less peak memory.
As a correctness check, the features produced by the timed implementation have cosine similarity of at least $0.99$ with the stored features used for ranking on two sampled responses from each setting.

\begin{table}[htbp]
\centering
\caption{\textbf{GRACE feature-extraction cost on one A100~80GB.}
Extraction time is estimated from $32$ timed responses per teacher and extrapolated to all $2{,}048$ responses per teacher; peak memory is measured directly.}
\label{tab:grace-timing}
\small
\begin{tabular}{l|r|r|r|r}
Setting & Full gradients & \LESSER{} & Speedup & \shortstack{Peak memory, GiB\\full / \LESSER{}} \\
\midrule
GSM8K, Llama-3.2-1B, $14$ teachers & $6.49$ h & $1.65$ h & $3.9\times$ & $19.6$ / $12.3$ \\
GSM8K, OLMo-2-1B, $10$ teachers & $5.34$ h & $0.99$ h & $5.4\times$ & $17.0$ / $7.0$ \\
MATH, Llama-3.2-3B, $10$ teachers & $14.68$ h & $3.09$ h & $4.8\times$ & $39.8$ / $24.2$ \\
\bottomrule
\end{tabular}
\end{table}

\section{Output-layer gradient feature and projection}
\label{app:derivations}

\subsection{Proof of \texorpdfstring{Proposition~\ref{prop:output-gradient}}{Proposition 1}}
\label{app:output-gradient-proof}
\begin{propositionrestated}
\propoutputgradientstatement
\end{propositionrestated}
\begin{proof}
Fix a token $t$ in the loss, and write $a_t=Wh_t\in\mathbb R^{|V|}$ for its logits and $\pi_t=\operatorname{softmax}(a_t)$ for the predicted distribution $\pi_\theta(\cdot\mid x,y_{<t})$. The loss at position $t$ is
\[
\ell_t=-\log\pi_t[y_t]=-a_t[y_t]+\log\sum_{v=1}^{|V|}\exp\bigl(a_t[v]\bigr).
\]
\emph{Step 1: gradient with respect to the logits.} For each vocabulary index $i$,
\[
\frac{\partial\ell_t}{\partial a_t[i]}=-e_{y_t}[i]+\frac{\exp\bigl(a_t[i]\bigr)}{\sum_{v=1}^{|V|}\exp\bigl(a_t[v]\bigr)}=\pi_t[i]-e_{y_t}[i],
\qquad\text{so}\qquad
\nabla_{a_t}\ell_t=\pi_t-e_{y_t}=r_t.
\]
\emph{Step 2: gradient with respect to $W_{\mathrm{out}}$.} The entry $(W_{\mathrm{out}})_{ij}$ enters only the logit $a_t[i]=\sum_{j'=1}^{d}(W_{\mathrm{out}})_{ij'}h_t[j']$, with $\partial a_t[i]/\partial (W_{\mathrm{out}})_{ij}=h_t[j]$. By the chain rule,
\[
\frac{\partial\ell_t}{\partial (W_{\mathrm{out}})_{ij}}=\frac{\partial\ell_t}{\partial a_t[i]}\,h_t[j]=r_t[i]\,h_t[j],
\qquad\text{so}\qquad
\nabla_{W_{\mathrm{out}}}\ell_t=r_th_t^\top.
\]
\emph{Step 3: sum over positions.} By \cref{eq:weighted-token-loss}, $L_p=\sum_t\alpha_t\ell_t$ with fixed weights $\alpha_t$, and the gradient is linear, so $g_p=\nabla_{W_{\mathrm{out}}}L_p=\sum_t\alpha_t\,r_th_t^\top$.
\end{proof}

\subsection{Projection}
\label{app:sketch}
Storing the output-layer gradient $g_p$ of Proposition~\ref{prop:output-gradient} for every pool example is impractical, so \LESSER{} stores a random projection of it, as in TRAK~\citep{park2023trak}. The matrices $\Pi_1\in\{\pm1/\sqrt m\}^{m\times|V|}$ and $\Pi_2\in\{\pm1/\sqrt n\}^{d\times n}$ have independent Rademacher entries, are drawn once per model, and are shared across pool and query examples. The feature is
\begin{equation}
\phi_p=\frac{\operatorname{vec}(\Pi_1g_p\Pi_2)}{\lVert \Pi_1g_p\Pi_2\rVert_F},\qquad \Pi_1g_p\Pi_2=\sum_t\alpha_t\,(\Pi_1r_t)(\Pi_2^\top h_t)^\top,
\label{eq:prod-sketch}
\end{equation}
so each token's residual and hidden state are projected before their outer product, without forming the $|V|\times d$ gradient. We use $m=64$ and $n=128$, giving $8{,}192$ dimensions, for SFT and the sketched RL comparison.

\FloatBarrier
\section{Additional Results: Supervised Fine-Tuning}
\label{app:sft-results}

\FloatBarrier
\subsection{SFT results across four models}
\label{app:full-grid}

We test whether the Llama-2 results extend across model families, using four models, five tasks, three budgets, and three training seeds. The full grid is reported in Tables~\ref{tab:main-results-k1000}--\ref{tab:main-results-k10000}.

\LESSER{} has a higher mean than Random on TyDiQA for every model and budget, by $0.5$ to $9.6$ points.
On TyDiQA, \LESSER{} is within $0.9$ points of \LESS{} or above it on Qwen3 and OLMo3 at every budget.
On the two Llama models, \LESS{} is higher on TyDiQA by $3.1$ to $4.4$ points at $k=1{,}000$, and the difference shrinks to $0.8$ to $1.2$ points at $k=10{,}000$.
On MMLU-Pro, the four methods lie within $1.3$ points of each other on every model at $k=5{,}000$ and $k=10{,}000$.

\begin{table*}[!htbp]
\centering
\caption{Held-out performance at $k=1{,}000$, mean $\pm$ sample standard deviation over three seeds; each task uses its target metric on a $0$--$100$ scale, higher is better.}
\label{tab:main-results-k1000}
\small\begin{tabular}{l|l|r|r|r|r}
Model & Task & \LESSER{} & \LESS{} & \RDS{} & Random \\
\midrule
Llama-3.2-3B & TyDiQA & $63.3{\scriptstyle\pm}0.5$ & $\mathbf{66.4{\scriptstyle\pm}0.2}$ & $45.7{\scriptstyle\pm}0.4$ & $53.7{\scriptstyle\pm}1.6$ \\
Llama-3.2-3B & MMLU-Pro & $22.1{\scriptstyle\pm}2.3$ & $22.8{\scriptstyle\pm}0.5$ & $\mathbf{25.3{\scriptstyle\pm}0.8}$ & $14.9{\scriptstyle\pm}8.9$ \\
Llama-3.2-3B & GSM8K & $32.3{\scriptstyle\pm}0.9$ & $\mathbf{34.2{\scriptstyle\pm}0.5}$ & $32.3{\scriptstyle\pm}0.7$ & $29.2{\scriptstyle\pm}0.7$ \\
Llama-3.2-3B & Codex & $\mathbf{61.3{\scriptstyle\pm}2.4}$ & $56.3{\scriptstyle\pm}3.0$ & $60.6{\scriptstyle\pm}1.7$ & $59.2{\scriptstyle\pm}7.8$ \\
Llama-3.2-3B & BBH & $48.0{\scriptstyle\pm}0.2$ & $47.6{\scriptstyle\pm}0.2$ & $\mathbf{48.7{\scriptstyle\pm}0.5}$ & $47.1{\scriptstyle\pm}0.5$ \\
\midrule
Llama-2-7B & TyDiQA & $50.2{\scriptstyle\pm}0.4$ & $\mathbf{54.6{\scriptstyle\pm}0.4}$ & $46.2{\scriptstyle\pm}0.6$ & $49.4{\scriptstyle\pm}0.7$ \\
Llama-2-7B & MMLU-Pro & $\mathbf{19.7{\scriptstyle\pm}0.0}$ & $19.2{\scriptstyle\pm}0.1$ & $19.6{\scriptstyle\pm}0.1$ & $16.3{\scriptstyle\pm}0.6$ \\
Llama-2-7B & GSM8K & $19.2{\scriptstyle\pm}0.3$ & $19.2{\scriptstyle\pm}0.7$ & $\mathbf{22.2{\scriptstyle\pm}0.7}$ & $15.0{\scriptstyle\pm}1.3$ \\
Llama-2-7B & Codex & $\mathbf{29.5{\scriptstyle\pm}0.4}$ & $25.0{\scriptstyle\pm}2.0$ & $27.7{\scriptstyle\pm}3.1$ & $26.4{\scriptstyle\pm}2.4$ \\
Llama-2-7B & BBH & $38.7{\scriptstyle\pm}0.4$ & $\mathbf{40.5{\scriptstyle\pm}0.1}$ & $40.3{\scriptstyle\pm}0.3$ & $39.1{\scriptstyle\pm}0.4$ \\
\midrule
Qwen3-4B-Base & TyDiQA & $70.7{\scriptstyle\pm}0.4$ & $\mathbf{71.6{\scriptstyle\pm}0.0}$ & $66.4{\scriptstyle\pm}0.6$ & $68.6{\scriptstyle\pm}1.5$ \\
Qwen3-4B-Base & MMLU-Pro & $48.5{\scriptstyle\pm}0.9$ & $46.7{\scriptstyle\pm}0.5$ & $\mathbf{49.3{\scriptstyle\pm}1.8}$ & $45.6{\scriptstyle\pm}2.8$ \\
Qwen3-4B-Base & GSM8K & $79.6{\scriptstyle\pm}0.2$ & $81.8{\scriptstyle\pm}0.2$ & $79.8{\scriptstyle\pm}0.5$ & $\mathbf{82.3{\scriptstyle\pm}0.2}$ \\
Qwen3-4B-Base & Codex & $90.3{\scriptstyle\pm}1.0$ & $\mathbf{92.1{\scriptstyle\pm}1.0}$ & $91.2{\scriptstyle\pm}1.4$ & $91.7{\scriptstyle\pm}1.4$ \\
Qwen3-4B-Base & BBH & $70.2{\scriptstyle\pm}0.2$ & $70.7{\scriptstyle\pm}0.2$ & $70.6{\scriptstyle\pm}0.2$ & $\mathbf{70.8{\scriptstyle\pm}0.3}$ \\
\midrule
OLMo3-7B & TyDiQA & $61.1{\scriptstyle\pm}0.8$ & $\mathbf{62.0{\scriptstyle\pm}0.5}$ & $55.3{\scriptstyle\pm}0.4$ & $57.8{\scriptstyle\pm}1.1$ \\
OLMo3-7B & MMLU-Pro & $\mathbf{38.6{\scriptstyle\pm}0.1}$ & $37.4{\scriptstyle\pm}0.0$ & $29.7{\scriptstyle\pm}8.3$ & $37.6{\scriptstyle\pm}1.6$ \\
OLMo3-7B & GSM8K & $73.3{\scriptstyle\pm}0.5$ & $74.2{\scriptstyle\pm}0.0$ & $\mathbf{75.0{\scriptstyle\pm}0.4}$ & $61.5{\scriptstyle\pm}14.2$ \\
OLMo3-7B & Codex & $\mathbf{84.0{\scriptstyle\pm}2.4}$ & $80.0{\scriptstyle\pm}0.4$ & $\mathbf{84.0{\scriptstyle\pm}3.1}$ & $80.6{\scriptstyle\pm}0.8$ \\
OLMo3-7B & BBH & $\mathbf{64.5{\scriptstyle\pm}0.1}$ & $64.4{\scriptstyle\pm}0.1$ & $63.0{\scriptstyle\pm}0.0$ & $64.4{\scriptstyle\pm}0.2$ \\
\bottomrule
\end{tabular}
\end{table*}

\begin{table*}[!htbp]
\centering
\caption{Held-out performance at $k=5{,}000$.}
\label{tab:main-results-k5000}
\small\begin{tabular}{l|l|r|r|r|r}
Model & Task & \LESSER{} & \LESS{} & \RDS{} & Random \\
\midrule
Llama-3.2-3B & TyDiQA & $64.3{\scriptstyle\pm}1.4$ & $\mathbf{66.3{\scriptstyle\pm}0.2}$ & $58.2{\scriptstyle\pm}0.9$ & $60.5{\scriptstyle\pm}0.2$ \\
Llama-3.2-3B & MMLU-Pro & $\mathbf{26.9{\scriptstyle\pm}0.4}$ & $26.1{\scriptstyle\pm}0.2$ & $\mathbf{26.9{\scriptstyle\pm}0.1}$ & $26.6{\scriptstyle\pm}0.1$ \\
Llama-3.2-3B & GSM8K & $34.6{\scriptstyle\pm}0.1$ & $34.5{\scriptstyle\pm}0.6$ & $\mathbf{40.0{\scriptstyle\pm}0.6}$ & $33.3{\scriptstyle\pm}0.9$ \\
Llama-3.2-3B & Codex & $57.4{\scriptstyle\pm}2.0$ & $58.1{\scriptstyle\pm}3.1$ & $\mathbf{61.0{\scriptstyle\pm}0.4}$ & $59.0{\scriptstyle\pm}1.0$ \\
Llama-3.2-3B & BBH & $48.5{\scriptstyle\pm}0.3$ & $49.6{\scriptstyle\pm}0.4$ & $\mathbf{49.7{\scriptstyle\pm}0.3}$ & $48.5{\scriptstyle\pm}0.5$ \\
\midrule
Llama-2-7B & TyDiQA & $51.7{\scriptstyle\pm}0.4$ & $\mathbf{52.9{\scriptstyle\pm}0.1}$ & $50.4{\scriptstyle\pm}0.4$ & $51.2{\scriptstyle\pm}0.2$ \\
Llama-2-7B & MMLU-Pro & $20.2{\scriptstyle\pm}0.3$ & $20.2{\scriptstyle\pm}0.3$ & $\mathbf{20.7{\scriptstyle\pm}0.1}$ & $20.0{\scriptstyle\pm}0.4$ \\
Llama-2-7B & GSM8K & $21.7{\scriptstyle\pm}0.5$ & $21.3{\scriptstyle\pm}0.5$ & $\mathbf{23.0{\scriptstyle\pm}0.4}$ & $16.7{\scriptstyle\pm}0.8$ \\
Llama-2-7B & Codex & $29.1{\scriptstyle\pm}2.4$ & $24.5{\scriptstyle\pm}2.1$ & $\mathbf{30.9{\scriptstyle\pm}1.6}$ & $27.3{\scriptstyle\pm}1.4$ \\
Llama-2-7B & BBH & $41.0{\scriptstyle\pm}0.4$ & $\mathbf{42.9{\scriptstyle\pm}0.4}$ & $41.8{\scriptstyle\pm}0.2$ & $40.3{\scriptstyle\pm}0.6$ \\
\midrule
Qwen3-4B-Base & TyDiQA & $\mathbf{73.4{\scriptstyle\pm}0.2}$ & $71.7{\scriptstyle\pm}0.3$ & $71.2{\scriptstyle\pm}0.5$ & $70.9{\scriptstyle\pm}0.7$ \\
Qwen3-4B-Base & MMLU-Pro & $48.7{\scriptstyle\pm}0.4$ & $49.2{\scriptstyle\pm}0.0$ & $\mathbf{50.0{\scriptstyle\pm}0.4}$ & $49.0{\scriptstyle\pm}0.2$ \\
Qwen3-4B-Base & GSM8K & $79.7{\scriptstyle\pm}0.8$ & $80.3{\scriptstyle\pm}0.7$ & $78.7{\scriptstyle\pm}0.3$ & $\mathbf{80.4{\scriptstyle\pm}0.5}$ \\
Qwen3-4B-Base & Codex & $90.3{\scriptstyle\pm}0.4$ & $92.1{\scriptstyle\pm}1.4$ & $\mathbf{92.3{\scriptstyle\pm}2.2}$ & $\mathbf{92.3{\scriptstyle\pm}2.2}$ \\
Qwen3-4B-Base & BBH & $69.5{\scriptstyle\pm}0.1$ & $\mathbf{71.5{\scriptstyle\pm}0.0}$ & $70.5{\scriptstyle\pm}0.6$ & $70.3{\scriptstyle\pm}0.6$ \\
\midrule
OLMo3-7B & TyDiQA & $63.5{\scriptstyle\pm}0.4$ & $\mathbf{63.6{\scriptstyle\pm}0.5}$ & $59.1{\scriptstyle\pm}0.2$ & $60.6{\scriptstyle\pm}0.5$ \\
OLMo3-7B & MMLU-Pro & $37.8{\scriptstyle\pm}0.1$ & $37.8{\scriptstyle\pm}0.4$ & $\mathbf{38.2{\scriptstyle\pm}0.2}$ & $37.6{\scriptstyle\pm}0.4$ \\
OLMo3-7B & GSM8K & $73.6{\scriptstyle\pm}0.4$ & $73.0{\scriptstyle\pm}1.8$ & $\mathbf{76.4{\scriptstyle\pm}0.8}$ & $28.7{\scriptstyle\pm}12.1$ \\
OLMo3-7B & Codex & $81.5{\scriptstyle\pm}1.7$ & $81.5{\scriptstyle\pm}0.8$ & $\mathbf{82.9{\scriptstyle\pm}4.6}$ & $82.0{\scriptstyle\pm}1.0$ \\
OLMo3-7B & BBH & $64.3{\scriptstyle\pm}0.2$ & $64.7{\scriptstyle\pm}0.2$ & $63.2{\scriptstyle\pm}0.2$ & $\mathbf{65.0{\scriptstyle\pm}0.2}$ \\
\bottomrule
\end{tabular}
\end{table*}

\begin{table*}[!htbp]
\centering
\caption{Held-out performance at $k=10{,}000$.}
\label{tab:main-results-k10000}
\small\begin{tabular}{l|l|r|r|r|r}
Model & Task & \LESSER{} & \LESS{} & \RDS{} & Random \\
\midrule
Llama-3.2-3B & TyDiQA & $66.0{\scriptstyle\pm}0.2$ & $\mathbf{67.2{\scriptstyle\pm}0.1}$ & $60.3{\scriptstyle\pm}0.1$ & $63.6{\scriptstyle\pm}0.4$ \\
Llama-3.2-3B & MMLU-Pro & $26.5{\scriptstyle\pm}0.2$ & $27.1{\scriptstyle\pm}0.4$ & $\mathbf{27.3{\scriptstyle\pm}0.1}$ & $26.9{\scriptstyle\pm}0.3$ \\
Llama-3.2-3B & GSM8K & $35.2{\scriptstyle\pm}0.4$ & $36.7{\scriptstyle\pm}0.9$ & $\mathbf{39.6{\scriptstyle\pm}0.3}$ & $35.2{\scriptstyle\pm}0.9$ \\
Llama-3.2-3B & Codex & $57.4{\scriptstyle\pm}0.7$ & $58.3{\scriptstyle\pm}2.1$ & $\mathbf{60.1{\scriptstyle\pm}2.0}$ & $58.3{\scriptstyle\pm}4.3$ \\
Llama-3.2-3B & BBH & $49.3{\scriptstyle\pm}0.5$ & $49.1{\scriptstyle\pm}0.6$ & $\mathbf{49.4{\scriptstyle\pm}0.4}$ & $48.6{\scriptstyle\pm}0.2$ \\
\midrule
Llama-2-7B & TyDiQA & $52.8{\scriptstyle\pm}0.4$ & $\mathbf{53.6{\scriptstyle\pm}0.1}$ & $50.7{\scriptstyle\pm}0.6$ & $52.1{\scriptstyle\pm}0.2$ \\
Llama-2-7B & MMLU-Pro & $\mathbf{21.2{\scriptstyle\pm}0.3}$ & $21.0{\scriptstyle\pm}0.2$ & $20.7{\scriptstyle\pm}0.3$ & $20.7{\scriptstyle\pm}0.2$ \\
Llama-2-7B & GSM8K & $22.2{\scriptstyle\pm}0.5$ & $22.6{\scriptstyle\pm}1.1$ & $\mathbf{24.2{\scriptstyle\pm}0.3}$ & $17.3{\scriptstyle\pm}1.0$ \\
Llama-2-7B & Codex & $30.4{\scriptstyle\pm}0.7$ & $26.6{\scriptstyle\pm}0.8$ & $29.7{\scriptstyle\pm}1.4$ & $\mathbf{30.6{\scriptstyle\pm}1.6}$ \\
Llama-2-7B & BBH & $41.7{\scriptstyle\pm}0.3$ & $\mathbf{43.1{\scriptstyle\pm}0.1}$ & $42.7{\scriptstyle\pm}0.0$ & $41.5{\scriptstyle\pm}0.4$ \\
\midrule
Qwen3-4B-Base & TyDiQA & $\mathbf{72.7{\scriptstyle\pm}0.1}$ & $72.2{\scriptstyle\pm}0.0$ & $70.8{\scriptstyle\pm}0.1$ & $70.7{\scriptstyle\pm}0.3$ \\
Qwen3-4B-Base & MMLU-Pro & $48.3{\scriptstyle\pm}0.4$ & $48.8{\scriptstyle\pm}0.3$ & $\mathbf{49.5{\scriptstyle\pm}0.3}$ & $48.5{\scriptstyle\pm}0.3$ \\
Qwen3-4B-Base & GSM8K & $78.7{\scriptstyle\pm}0.6$ & $\mathbf{80.7{\scriptstyle\pm}0.6}$ & $79.2{\scriptstyle\pm}0.4$ & $79.7{\scriptstyle\pm}0.6$ \\
Qwen3-4B-Base & Codex & $91.9{\scriptstyle\pm}1.2$ & $\mathbf{92.8{\scriptstyle\pm}1.7}$ & $91.9{\scriptstyle\pm}1.2$ & $92.3{\scriptstyle\pm}1.7$ \\
Qwen3-4B-Base & BBH & $69.7{\scriptstyle\pm}0.5$ & $\mathbf{71.7{\scriptstyle\pm}0.1}$ & $70.2{\scriptstyle\pm}0.2$ & $69.8{\scriptstyle\pm}0.3$ \\
\midrule
OLMo3-7B & TyDiQA & $63.6{\scriptstyle\pm}0.2$ & $\mathbf{64.4{\scriptstyle\pm}0.1}$ & $60.1{\scriptstyle\pm}0.5$ & $61.0{\scriptstyle\pm}0.5$ \\
OLMo3-7B & MMLU-Pro & $37.4{\scriptstyle\pm}0.4$ & $\mathbf{38.7{\scriptstyle\pm}0.2}$ & $38.1{\scriptstyle\pm}1.0$ & $38.1{\scriptstyle\pm}0.3$ \\
OLMo3-7B & GSM8K & $74.7{\scriptstyle\pm}0.2$ & $74.6{\scriptstyle\pm}0.1$ & $\mathbf{76.1{\scriptstyle\pm}0.6}$ & $70.0{\scriptstyle\pm}6.7$ \\
OLMo3-7B & Codex & $\mathbf{82.9{\scriptstyle\pm}2.6}$ & $79.1{\scriptstyle\pm}0.7$ & $81.8{\scriptstyle\pm}0.7$ & $82.7{\scriptstyle\pm}2.1$ \\
OLMo3-7B & BBH & $65.0{\scriptstyle\pm}0.5$ & $65.0{\scriptstyle\pm}0.2$ & $63.4{\scriptstyle\pm}0.1$ & $\mathbf{65.1{\scriptstyle\pm}0.1}$ \\
\bottomrule
\end{tabular}
\end{table*}

\FloatBarrier
\subsection{Query-loss curves across models and tasks}
\label{app:bin-loss}

This subsection gives the detailed results of the similarity-bin diagnostic in \cref{sec:rq1-sft}. Figure~\ref{fig:bin-loss-curves} shows Llama-2; Figures~\ref{fig:bin-loss-grid}--\ref{fig:bin-loss-olmo3-7b} show the other models.

Higher similarity generally predicts lower query loss: mean bin-index Spearman correlation is $0.714$ for \LESSER{}, $0.887$ for \LESS{}, and $0.058$ for \RDS{}. \LESSER{} has positive correlation in $19$ of $20$ pairs and its closest bin has the lowest loss in $18$. OLMo3/MMLU-Pro is the correlation exception, with $\rho=-0.030$.

\begin{figure}[!htbp]
\centering
\includegraphics[width=\linewidth]{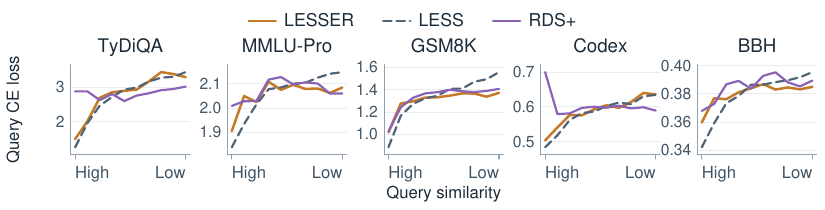}
\caption{\textbf{Llama-3.2-3B: query loss across similarity bins.} Each bin uses one training run; bin $0$ is closest to the queries.}
\label{fig:bin-loss-grid}
\end{figure}

\begin{figure}[!htbp]
\centering
\includegraphics[width=\linewidth]{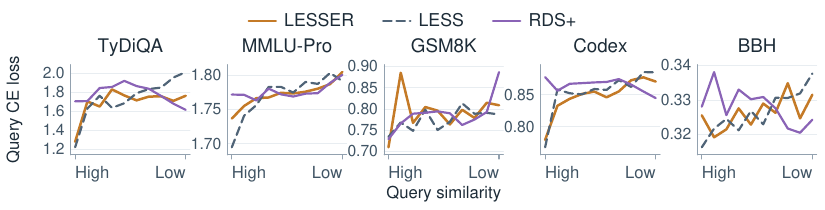}
\caption{\textbf{Qwen3-4B-Base: query loss across similarity bins.} Each bin uses one training run; bin $0$ is closest to the queries.}
\label{fig:bin-loss-qwen3-4b-base}
\end{figure}

\begin{figure}[!htbp]
\centering
\includegraphics[width=\linewidth]{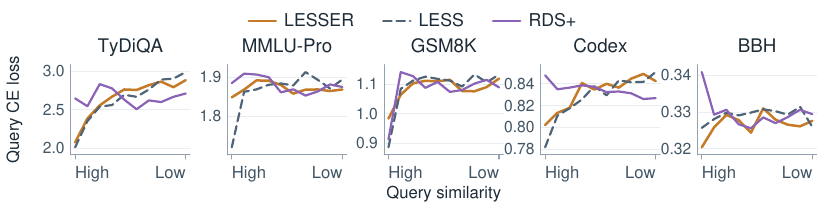}
\caption{\textbf{OLMo3-7B: the closest \LESSER{} bin has the lowest query loss on all five tasks.} Loss is not monotone across bins on MMLU-Pro. Each bin uses one training run.}
\label{fig:bin-loss-olmo3-7b}
\end{figure}

\FloatBarrier
\section{GRACE teacher rankings}
\label{app:grace-controls}

\noindent\Cref{tab:current-grace} gives the detailed results of the GRACE experiment in \cref{sec:grace}. With \LESSER{} features, GRACE ranks the teachers nearly as it does with full gradients and picks the same top-ranked teacher in all three settings.

\begin{table}[H]
\centering
\caption{Teacher-ranking agreement between \LESSER{} and full-gradient GRACE. Spearman correlation uses the complete ranking.}
\label{tab:current-grace}
\small\begin{tabular}{llrcl}
Task & Student & Teachers & Spearman & Top-ranked teacher \\
\midrule
GSM8K & Llama-3.2-1B & $14$ & $0.938$ & Qwen2.5-1.5B-Instruct \\
GSM8K & OLMo-2-1B & $10$ & $0.964$ & Qwen2.5-3B-Instruct \\
MATH & Llama-3.2-3B & $10$ & $0.976$ & Qwen2.5-3B-Instruct \\
\bottomrule
\end{tabular}
\end{table}

\FloatBarrier
\section{Overlap of different selection algorithms in SFT}
\label{app:utility-results}

\FloatBarrier
\label{app:selection-overlap}
Table~\ref{tab:selection-overlap} compares the saved \LESSER{} and released \LESS{} selections on all four models. Each set contains $5{,}000$ unique indices in the same $197{,}196$-example Tulu pool. Jaccard similarity is $|S_{\LESSER}\cap S_{\LESS}|/|S_{\LESSER}\cup S_{\LESS}|$; it measures item overlap.

\begin{table}[htbp]
\centering
\caption{\LESSER{} and \LESS{} select substantially different examples. Each cell gives the number of shared examples and the Jaccard similarity.}
\label{tab:selection-overlap}
\resizebox{\linewidth}{!}
{\begin{tabular}{l|r|r|r|r|r|r}
Model & TyDiQA & MMLU-Pro & GSM8K & Codex & BBH & Mean Jaccard \\
\midrule
Llama-2-7B & $491$ / $0.052$ & $505$ / $0.053$ & $324$ / $0.033$ & $214$ / $0.022$ & $228$ / $0.023$ & $0.037$ \\
Llama-3.2-3B & $407$ / $0.042$ & $433$ / $0.045$ & $441$ / $0.046$ & $266$ / $0.027$ & $186$ / $0.019$ & $0.036$ \\
Qwen3-4B-Base & $594$ / $0.063$ & $173$ / $0.018$ & $337$ / $0.035$ & $269$ / $0.028$ & $143$ / $0.015$ & $0.032$ \\
OLMo3-7B & $1{,}004$ / $0.112$ & $190$ / $0.019$ & $250$ / $0.026$ & $476$ / $0.050$ & $182$ / $0.019$ & $0.045$ \\
\bottomrule
\end{tabular}}
\end{table}

\FloatBarrier

\FloatBarrier

\FloatBarrier

\FloatBarrier

\section{Details of the analysis in Section~\ref{sec:analysis-new}}
\label{app:new-finite}

\subsection{Per-example scores}
\label{app:local-utility}

\paragraph{Motivation.}
We first ask how much of the full-gradient selection signal is already present in the output layer at the level of individual examples.
Let $L_Q$ be the aggregate loss on the query set $\mathcal Q$, with gradient $G_Q=\nabla_\theta L_Q$.
For a candidate $p$ with gradient $G_p$, we let local utility be the reduction in query loss from a small SGD step on $p$, normalized by the size of that step:
\begin{equation}
U(p)=\frac{L_Q(\theta)-L_Q(\theta-\eta G_p)}
{\eta\lVert G_p\rVert}
\approx
\frac{G_Q^\top G_p}{\lVert G_p\rVert}.
\label{eq:onestep-value}
\end{equation}
Thus, the cosine between candidate and query gradients provides a first-order measure of how useful an individual candidate is for the query set.

\paragraph{Setup.}
We evaluate these per-example scores on $20$ model--task pairs.
For each pair, we sample up to $3{,}000$ candidates from the pool and, for the two 7B models, keep only examples of at most $1{,}024$ tokens, leaving $1{,}903$--$3{,}000$ candidates per setting.
For each example, we compute fp32 gradients of all native parameters at the base model.
To measure how the ranking changes as more layers of the model are included, Figure~\ref{fig:layer-curve} starts with the output layer and progressively adds decoder blocks from last to first, including each block's attention, MLP, and normalization parameters.

\paragraph{Findings.}
Output-layer and full-gradient cosine scores have a median Spearman correlation of $0.52$ across the $20$ model--task pairs, with per-model medians ranging from $0.36$ to $0.72$.
Adding more layers steadily closes this gap: the median correlation with the full-gradient ranking rises from $0.52$ for the output layer alone to $0.65$ with half of the decoder blocks and to $1.00$ when all decoder blocks are included.
Thus, output-layer gradients capture a substantial but incomplete part of the per-example ordering; recovering the full-gradient ranking generally requires information from much of the network.

\begin{figure}[htbp]
\centering
\includegraphics[width=0.70\linewidth]{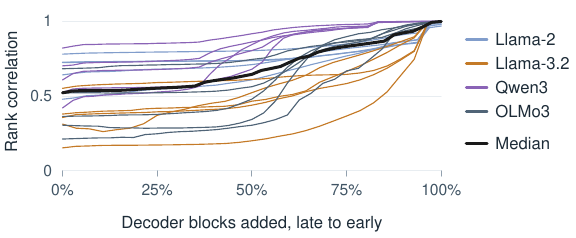}
\caption{\textbf{Per-example rankings depend on gradients throughout the network.}
Spearman correlation with the full-gradient cosine ranking as decoder blocks are added to the output layer from last to first.
Candidates are ranked by their cosine with the query gradient over the included parameters.
Thin lines denote model--task pairs, colored by model, and the black line shows the median.}
\label{fig:layer-curve}
\end{figure}

To characterize what is lost by using only the output layer, we decompose the full-gradient score into contributions from the output layer and the remaining parameters. Expand
\[
G_p
=
\begin{bmatrix}
G_p^{\mathrm{out}}\\
G_p^{\mathrm{body}}
\end{bmatrix},
\qquad
G_Q
=
\begin{bmatrix}
G_Q^{\mathrm{out}}\\
G_Q^{\mathrm{body}}
\end{bmatrix},
\]
\[
G_Q^\top G_p
=
\left(G_Q^{\mathrm{out}}\right)^\top G_p^{\mathrm{out}}
+
\left(G_Q^{\mathrm{body}}\right)^\top G_p^{\mathrm{body}}.
\]
Since $\lVert G_Q\rVert$ is constant across candidates, have
\[
\frac{G_Q^\top G_p}{\lVert G_p\rVert}
=
\underbrace{
\frac{
\left(G_Q^{\mathrm{out}}\right)^\top G_p^{\mathrm{out}}
}{
\lVert G_p\rVert
}
}_{H(p)}
+
\underbrace{
\frac{
\left(G_Q^{\mathrm{body}}\right)^\top G_p^{\mathrm{body}}
}{
\lVert G_p\rVert
}
}_{B(p)}.
\]
So $H(p)$ and $B(p)$ are the relative contributions of output-layer and body layers.

Across model--task pairs, $H(p)$ and $B(p)$ have a median correlation of $0.36$, indicating that the remaining layers do not simply reproduce the output-layer contribution. However, candidates favored by the output layer also tend to receive favorable contributions from the rest of the network. Let $\mathcal{T}$ denote the top $10\%$ of candidates under the output-layer cosine score, and define
\[
\Delta_H
=
\mathbb{E}_{p\in\mathcal{T}}[H(p)]
-
\mathbb{E}_{p\in\mathcal{D}}[H(p)],
\qquad
\Delta_B
=
\mathbb{E}_{p\in\mathcal{T}}[B(p)]
-
\mathbb{E}_{p\in\mathcal{D}}[B(p)].
\]
We find $\Delta_B>0$ in $19$ of $20$ model--task pairs, with median
$
\frac{\Delta_B}{\Delta_H}=0.82.
$
So candidates favored by the output layer are usually also favored by the rest of the network.

\subsection{Batch gradient similarity between selections}
\label{app:gradient-similarity}

\paragraph{Setup.}
We compare the full-model gradients induced by batches from different selections at the base model $\theta_0$.
Let $G_S$ be the gradient of the recipe's token-averaged loss on batch $S$, and let $\mu$ be the expected gradient of a uniformly sampled batch of the same size.
To measure alignment beyond this shared random-batch direction, we use the centered batch-gradient alignment between the batches $S$ of one method and the batches $S'$ of another
\begin{equation}
\frac{
\operatorname{avg}\,(G_S-\mu)^\top(G_{S'}-\mu)
}{
\sqrt{
\operatorname{avg}\lVert G_S-\mu\rVert^2
\;\operatorname{avg}\lVert G_{S'}-\mu\rVert^2
}
},
\label{eq:gradient-similarity}
\end{equation}
where the numerator averages over all pairs of a batch $S$ and a batch $S'$, and each squared norm averages over the batches of one sample.
Because the population mean random-batch gradient
$\mu=\mathbb{E}[G_S]$ is unknown, we estimate it from the random samples that are not part of the comparison, described below. When neither sample is random, two of the four random samples give $\hat\mu_1$ and the other two give $\hat\mu_2$, averaged over the three ways to split them; when one sample is random, any two of the other three give $\hat\mu_1$ and $\hat\mu_2$, averaged over the choices.
We then replace
$(G_S-\mu)^\top(G_{S'}-\mu)$ with
$(G_S-\hat\mu_1)^\top(G_{S'}-\hat\mu_2)$,
so that no random example contributes to both mean estimates, and estimate each squared norm as $(G_S-\hat\mu_1)^\top(G_S-\hat\mu_2)$.

For each of the $20$ model--task pairs, we sample $1{,}024$ examples from the $k=5{,}000$ selections of \LESSER{}, \LESS{}, \RDS{}, and Random. We also draw three additional uniform samples of $1{,}024$ examples from the candidate pool, giving four random samples in total.

We use these samples to construct two baselines in Figure~\ref{fig:alignment-base}. The Random baseline compares each random sample with \LESS{}, while the \LESS{} baseline compares two disjoint batches drawn from the same \LESS{} selection.

\paragraph{Findings.}
Raw batch gradients can be strongly aligned even under random sampling: at batch size $128$, two random batches have median cosine $0.81$ across the $20$ model--task pairs.
Thus, raw cosine largely reflects a direction shared by typical training batches, rather than the effect of selection itself.
We therefore center each batch gradient by the expected random-batch gradient.

Under this centered measure, \LESSER{} batches remain substantially aligned with \LESS{} batches.
At batch size $128$, their similarity is positive in $17$ of $20$ pairs and typically exceeds the \RDS{} and random baselines (Figure~\ref{fig:alignment-base}).
Across batch sizes $64$--$512$, the median over model--task pairs of the ratio between the \LESSER{}--\LESS{} similarity and the similarity between two disjoint \LESS{} batches is $0.58$--$0.66$.

\subsection{Training trajectories}
\label{app:trajectories}

\paragraph{Setup.}
For each of the $20$ model--task pairs, we train for $80$ AdamW steps with batch size $128$ on the $k=5{,}000$ selected examples, following Appendix~\ref{app:sft-protocol}.
We use two training seeds each for \LESSER{} and \LESS{} and one run for each of two random selections.
At every step, we record the query loss.

Figure~\ref{fig:trajectories-loss} shows the query-loss curves of every model--task pair; Figure~\ref{fig:training-directions} summarizes these curves by the median share of the largest reduction reached at each step.

\begin{figure}[htbp]
\centering
\includegraphics[width=\linewidth]{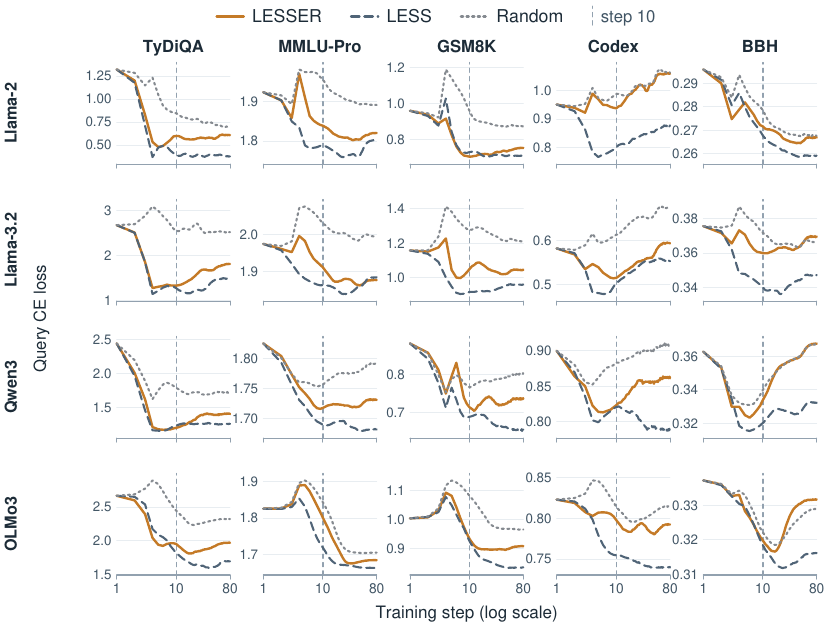}
\caption{\textbf{Query loss across training steps.} Each panel is one model--task pair.}
\label{fig:trajectories-loss}
\end{figure}

\subsection{RL and GRACE statistics}
\label{app:rl-grace}

\paragraph{RL.}
We analyze the first selection round of the Countdown experiment, restricting to the $1{,}835$ of $5{,}120$ candidate problems with nonzero advantage.
Each problem aggregates $16$ sampled responses, and we compare its \LESSER{} and full-gradient scores at the base policy.
The two rankings have Spearman correlation $\rho=0.69$.
At the selection budget, they select $69\%$ of the same problems, a Jaccard similarity of $0.53$ against $0.075$ for two random selections.

\paragraph{GRACE.}
We analyze the $14$ teachers in the GSM8K experiment with the Llama-3.2-1B student.
For each teacher, GRACE summarizes the student's gradients on $2{,}048$ teacher responses using statistics of their spectrum and pairwise similarities.
Replacing full gradients with output-layer gradients largely preserves these teacher-level quantities: across teachers, spectral entropy and effective rank have Spearman correlations $0.94$ and $0.93$, while the remaining aggregate statistics have correlations $0.75$--$0.98$.
This agreement does not hold at the response level: the $64$ most central responses under the two representations have Jaccard similarity only $0.12$.

\end{document}